\documentclass[final]{siamltex1213}
\usepackage{graphicx}
\usepackage[caption=false]{subfig}
\usepackage{algorithm}
\usepackage{algorithmic}
\usepackage{amssymb}
\usepackage{amsmath} 
\usepackage{array}
\usepackage{amsfonts}
\newtheorem{remark}{\it Remark}
\newcolumntype{L}[1]{>{\raggedright\let\newline\\\arraybackslash\hspace{0pt}}m{#1}}
\newcommand{\R}{\ensuremath{\mathbb{R}}}
\newcommand{\Diag}{\mathop{\mathrm{Diag}}\nolimits}
\newcommand{\tr}{\mathop{\mathrm{Tr}}\nolimits}
\newcommand*{\defeq}{\mathrel{\vcenter{\baselineskip0.5ex \lineskiplimit0pt
                     \hbox{\scriptsize.}\hbox{\scriptsize.}}}
                     =}

\title{Stable Camera Motion Estimation Using Convex Programming}

\author{Onur~\"{O}zye\c{s}\.{i}l\footnotemark[2]
\and Amit~Singer\footnotemark[3]
\and Ronen~Basri\footnotemark[4]}

\begin{document}
\maketitle
\newcommand{\slugmaster}{%
\slugger{siims}{xxxx}{xx}{x}{x--x}}

\renewcommand{\thefootnote}{\fnsymbol{footnote}}

\footnotetext[2]{Program in Applied and Computational Mathematics (PACM), Princeton University, Princeton, NJ 08544-1000, USA (\email{oozyesil@math.princeton.edu}).}
\footnotetext[3]{Department of Mathematics and PACM, Princeton University, Princeton, NJ 08544-1000, USA (\email{amits@math.princeton.edu}).}
\footnotetext[4]{Department of Computer Science and Applied Mathematics, Weizmann Institute of Science, Rehovot, 76100, ISRAEL (\email{ronen.basri@weizmann.ac.il}).}

\renewcommand{\thefootnote}{\arabic{footnote}}

\begin{abstract}
We study the inverse problem of estimating $n$ locations $\mathbf{t}_1, \mathbf{t}_2, \ldots, \mathbf{t}_n$ (up to global scale, translation and  negation) in $\R^d$ from noisy measurements of a subset of the (unsigned) pairwise lines that connect them, that is, from noisy measurements of $\pm \frac{\mathbf{t}_i - \mathbf{t}_j}{\|\mathbf{t}_i - \mathbf{t}_j \|_2}$ for some pairs $(i,j)$ (where the signs are unknown). This problem is at the core of the structure from motion (SfM) problem in computer vision, where the $\mathbf{t}_i$'s represent camera locations in $\R^3$. The noiseless version of the problem, with exact line measurements, has been considered previously under the general title of parallel rigidity theory, mainly in order to characterize the conditions for unique realization of locations. For noisy pairwise line measurements, current methods tend to produce spurious solutions that are clustered around a few locations. This sensitivity of the location estimates is a well-known problem in SfM, especially for large, irregular collections of images.\\
\indent In this paper we introduce a semidefinite programming (SDP) formulation, specially tailored to overcome the clustering phenomenon. We further identify the implications of parallel rigidity theory for the location estimation problem to be well-posed, and prove exact (in the noiseless case) and stable location recovery results. We also formulate an alternating direction method to solve the resulting semidefinite program, and provide a distributed version of our formulation for large numbers of locations. Specifically for the camera location estimation problem, we formulate a pairwise line estimation method based on robust camera orientation and subspace estimation. Lastly, we demonstrate the utility of our algorithm through experiments on real images.
\end{abstract}



\pagestyle{myheadings}
\thispagestyle{plain}
\markboth{O.~\"{O}ZYE\c{S}\.{I}L, A.~SINGER AND R.~BASRI}{STABLE CAMERA MOTION ESTIMATION}

\section{Introduction}  
Global positioning of $n$ objects from partial information about their relative locations is prevalent in many applications spanning fields such as sensor network localization~\cite{SNLSDP1,TubaishatMadriaSensors,AmitMihailSensors,ErenNetwork2}, structural biology~\cite{HendricksonBio}, and computer vision~\cite{HartleyBook,BATL2}. A well-known instance that attracted much attention from both the theoretical and algorithmic perspectives is that of estimating the locations $\mathbf{t}_1, \mathbf{t}_2, \ldots, \mathbf{t}_n \in \R^d$ from their pairwise Euclidean distances $\|\mathbf{t}_i - \mathbf{t}_j\|_2$. In this case, the large body of literature in rigidity theory (cf.~\cite{AspnesSurvey,SNLSDP}) provides conditions under which the localization is unique given a set of noiseless distance measurements. Also, much progress has been made with algorithms that estimate positions from noisy distances, starting with classical multidimensional scaling~\cite{SchoenbergCMDS} to the more recent semidefinite programming (SDP) approaches (see, e.g.,~\cite{SNLSDP1,SNLSDP2}).
\newline\indent Here we consider a different global positioning problem, in which the locations $\mathbf{t}_1, \ldots, \mathbf{t}_n$ need to be estimated from a subset of (potentially noisy) measurements of the pairwise lines that connect them (see Figure~\ref{fig:LineEstInstance} for a noiseless instance). The line connecting $\mathbf{t}_i$ and $\mathbf{t}_j$ is identified with the rank-$1$ projection matrix $\Gamma_{ij}$ defined by
\begin{equation}
\label{eq:GammaMats}
\Gamma_{ij} = (\mathbf{t}_i - \mathbf{t}_j)(\mathbf{t}_i - \mathbf{t}_j)^T / \|\mathbf{t}_i - \mathbf{t}_j\|_2^2 \\
\end{equation}
Notice that there is no available information about the Euclidean distances between the points. The entire information of pairwise lines is represented as a measurement graph $G_t = (V_t, E_t)$, where the $i$'th node in $V_t = \{1,2,\ldots,n\}$ corresponds to the location $\mathbf{t}_i$ and each edge $(i,j)\in E_t$ is endowed with the corresponding projection matrix $\Gamma_{ij}$.  
\begin{figure}[!htbp]
\begin{center}
   \includegraphics[trim=0cm 0cm 0cm 0cm, clip=true, width=0.65\linewidth]{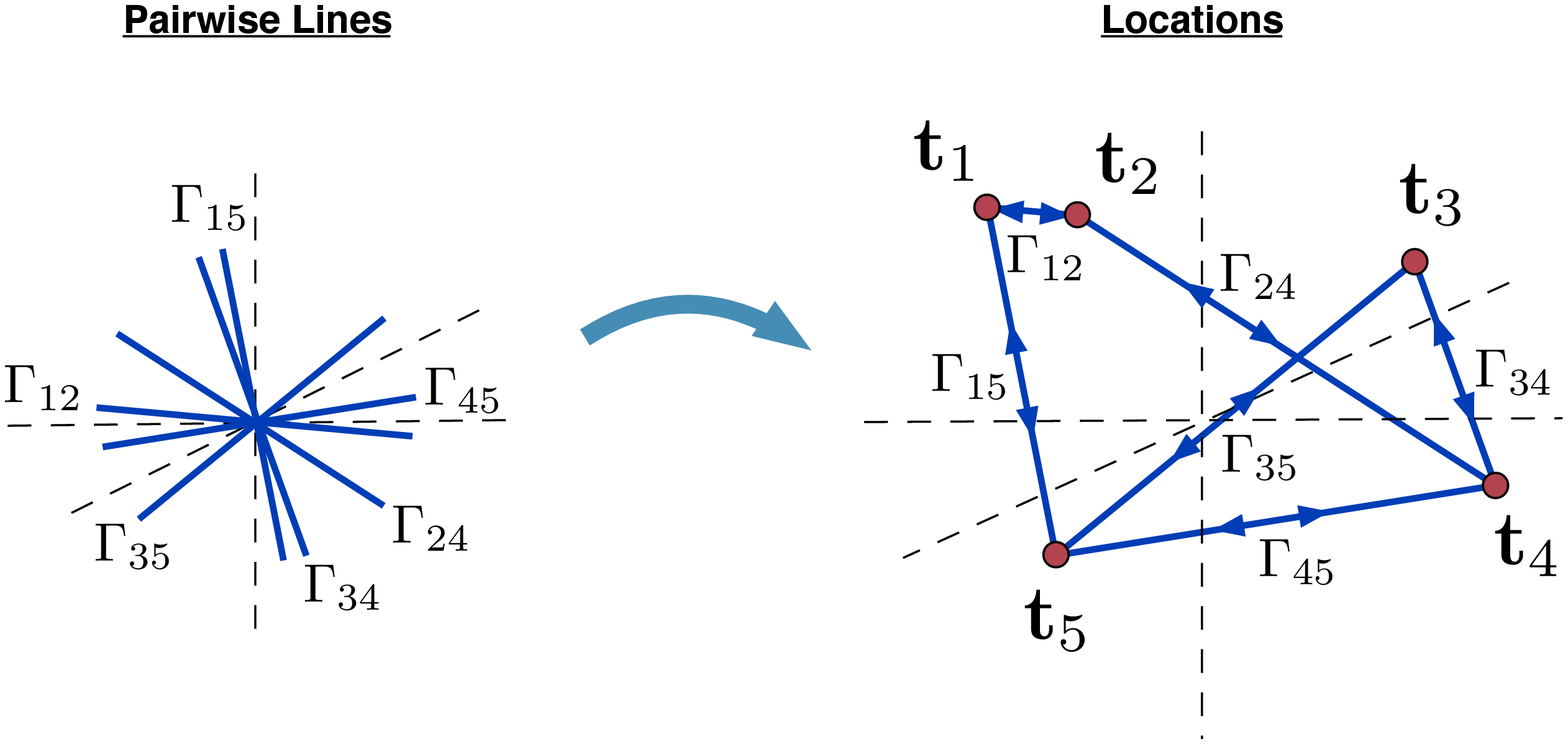}
\end{center}
   \caption{A (noiseless) instance of the line estimation problem in $\R^3$, with $n=5$ locations and $m=6$ pairwise lines.\label{fig:LineEstInstance}}
\end{figure}
\newline\indent The noiseless version of this problem (i.e., realization of locations from exact line measurements~(\ref{eq:GammaMats})) was previously studied in several different contexts such as discrete geometry, sensor network localization, and robotics, under various formulations (see~\cite{WhiteleyMatroidBook, WhiteleyMatroid, ErenNetwork, ErenNetwork2, ServatiusWhiteleyCAD}). The concepts and the results for the noiseless case, to which we refer here as {\em parallel rigidity theory}, are aimed at characterizing the conditions for the existence of a unique realization of the locations $\mathbf{t}_i$ (of course, up to global translation, scaling and negation of the locations $\mathbf{t}_i$, since the pairwise lines $\Gamma_{ij}$ are invariant under these transformations). 
\newline\indent However, location estimation from (potentially) noisy line measurements did not receive much attention previously. The camera location estimation part of the structure from motion (SfM) problem in computer vision (see, e.g.,~\cite{HartleyBook}), where $\mathbf{t}_i$'s represent the camera locations in $\R^3$, is an important example of the abstract problem with noisy measurements. To the best of our knowledge, a structured formulation (in terms of the pairwise lines) of the camera location estimation problem and its relation to the existing results of parallel rigidity theory (characterizing conditions for {\em well-posed} instances of the problem) were not considered previously. We now give more details on the camera location estimation problem and the existing techniques for its solution. \\ \\
\noindent {\bf Camera Location Estimation in SfM:} Structure from motion (SfM) (depicted in Figure~\ref{fig:SfMProblem}) is the problem of recovering a $3$D structure by estimating the camera motion corresponding to a collection of $2$D images (cf. \S\ref{sec:CamMotEst} for technical details). 
\begin{figure}[!htbp]
\begin{center}
   \includegraphics[trim=0cm 0cm 0cm 0cm, clip=true, width=0.7\linewidth]{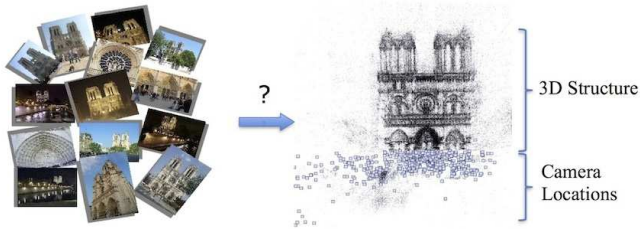}
\end{center}
   \caption{The structure from motion (SfM) problem.\label{fig:SfMProblem}}
\end{figure}
Classically, SfM is solved in three stages: $(1)$ Feature point matching between image pairs (as in~Figure~\ref{fig:CorrPtsEx}) and relative pose estimation of camera pairs based on extracted feature points  $(2)$ Estimation of camera motion, i.e. global camera orientations and locations, from relative poses $(3)$ $3$D structure recovery based on estimated camera motion by reprojection error minimization (e.g. bundle adjustment of~\cite{BundleAdjustment}). Although the first and the third stages are relatively well understood and there exist accurate and efficient algorithms in the literature for these stages, existing methods for camera motion estimation, and specifically for the camera location estimation part, are sensitive to errors that result from mismatched feature correspondences. Among the widely used methods are incremental approaches (e.g.~\cite{SnavelyRome,SnavelyData,SnavelySkeletal,VisualSfM,ZhangIncremental,HavlenaGroups,FurukawaPartial}), that integrate images to the estimation process one by one or in small groups. These incremental methods usually result in accumulation of estimation errors at each step and also require several applications of bundle adjustment for improved accuracy, hence leading to computational inefficiency. Alternatively, global estimation of the camera motion, i.e. solving for all camera orientations and/or locations simultaneously, can potentially yield a more accurate estimate. A generally adapted procedure to prevent high computational costs is to estimate the camera motion and the $3$D structure separately. In principle, considering the large reduction in the number of variables, camera motion estimation can be performed jointly for all images, followed by $3$D structure recovery using a single instance of reprojection error minimization. Obviously, such a procedure presumes a stable and efficient motion estimation algorithm. Most of the existing methods for motion estimation perform orientation and location estimation separately. Orientation estimation turns out to be a relatively well-posed problem and there exist several efficient and stable methods, e.g.~\cite{MartinecRotations,MicaAmitSfM,HartleyRotations,GovinduLieAlg,TronVidalDist}. On the other hand, global estimation of camera locations, specifically for large, unordered sets of images, usually suffers from instability to errors in feature correspondences (resulting in solutions clustered around a few locations, as for~\cite{MicaAmitSfM, BATL2, MultiLinear}), sensitivity to outliers in pairwise line measurements (e.g., the $\ell_{\infty}$ approach of~\cite{HartleySim, KahlHartley}) and susceptibility to local solutions in non-convex formulations (e.g.,~\cite{GovinduLieAlg}). Hence, a well-formulated, robust, efficient method for global location estimation (scalable to large sets of images) with provable convergence and stability guarantees is of high value.\\
\indent Early works of~\cite{GovinduEarlyL2,BATL2} on location estimation reduce the problem to that of solving a set of linear equations that originate from the pairwise line measurements. Finding solutions to these linear systems can be done in a computationally efficient manner. However, these solutions have been empirically observed to be sensitive to errors in the pairwise line measurements. The sensitivity of such solutions is expressed by the tendency of the estimated locations to cluster, regardless of the true locations (cf. Figure~\ref{fig:NotreL2Cluster} for such a clustering solution for a real data set, and also the optimization problem~(\ref{eq:L2TransEst}) and the following discussion). The multistage linear method of~\cite{MultiLinear} attempts to resolve this issue by first performing pairwise reconstructions, then registering these in pairs, and finally aligning them by estimating relative scales and locations. Nonetheless, this approach does not produce satisfactory results in terms of estimation accuracy. Another interesting and efficient method is the Lie algebraic averaging approach of~\cite{GovinduLieAlg}. However, this non-convex method is susceptible to convergence to local optima. The spectral formulation of~\cite{MicaAmitSfM}, which is based on a novel decomposition of the essential matrix and is similar to~\cite{GovinduEarlyL2,BATL2} in its formulation of the problem, yields an efficient linear solution for the locations, though, it also suffers from spurious clustered solutions. Yet another formulation related to our work is the quasi-convex formulation of~\cite{HartleySim}, which relies on (iteratively) optimizing a functional of the $\ell_\infty$ norm and requires the estimation of the signed directions of the pairwise lines, i.e. knowledge of $\frac{\mathbf{t}_i - \mathbf{t}_j}{\|\mathbf{t}_i - \mathbf{t}_j\|_2}$. However, $\ell_{\infty}$ norm is highly susceptible to outliers, resulting in unsatisfactory solutions for real image sets. Also, the requirement for estimating signed line directions may introduce additional error to the problem. A similar idea is employed in~\cite{KahlHartley} to simultaneously estimate the $3$D structure, which exhibits the same difficulties. There are also various works that aim to improve the high sensitivity and inefficiency of these quasi-convex methods (see, e.g.,~\cite{MartinecRotations,OlssonKahlEfficient}). Another method requiring the estimation of the signed directions of the pairwise lines is studied in~\cite{TronVidalDist}. In contrast to the sensitivity of the quasi-convex method of~\cite{HartleySim} to outliers, the method in~\cite{TronVidalDist} is based on optimizing a functional of the $\ell_2$ norm, and hence produces more accurate location estimates. Additionally, \cite{AddRef1} introduces a framework based on classical rigidity theory (involving pairwise distance information), which aims to identify rigid instances of the joint motion and structure estimation problem. Also, an SDP approach is introduced in~\cite{AddRef1} in order to jointly estimate motion and structure from noisy feature correspondences. We note that our work is significantly different from~\cite{AddRef1}: While we use parallel rigidity,~\cite{AddRef1} employs classical rigidity, leading to completely different SDP formulations.
\begin{figure}[!htbp]
\begin{center}
   \includegraphics[trim=2cm 1cm 1.4cm 0.75cm, clip=true, width=0.6\linewidth]{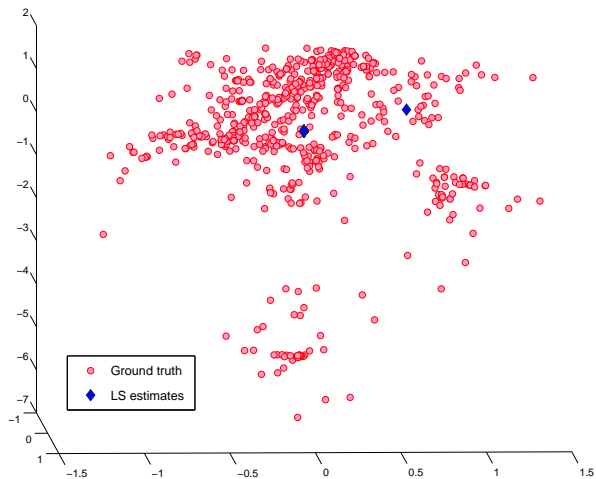}
\end{center}
   \caption{A $2$D snapshot of the collapsing solution of the least squares (LS) method~{\rm\cite{MicaAmitSfM,BATL2}}, for the Notre-Dame data set from~{\rm\cite{SnavelyData}}. The solution of~{\rm\cite{SnavelyData}} is taken as the ground truth.\label{fig:NotreL2Cluster}}
\end{figure}
\\ \\
\noindent {\bf Main Contributions and Broader Context:} In this paper we make the following principal contributions for the problem of location estimation from pairwise lines: 
\begin{romannum} \item The main contribution of our work is the introduction of a new semidefinite relaxation (SDR) formulation for estimating locations from pairwise line measurements. Empirically, we observe that solutions obtained by the SDR do not suffer from the clustering phenomenon of location estimates that plague many of the existing algorithms.
\item To quantify the performance of our SDR formulation, we prove exact (in the noiseless case, cf. Proposition~\ref{propo:ExactRecovery}) and stable (in the noisy case, cf. Theorem~\ref{thm:WeakStability} and Corollary~\ref{corol:LocationStability}) location recovery results. We also provide a provably convergent and efficient alternating direction method to solve our SDR. 
\item We provide a distributed version (scalable to a large number of locations) of the SDR formulation, based on spectral partitioning and convex programming. Additionally, we prove exact location recovery (in the noiseless case, cf. Proposition~\ref{propo:ExactRecoveryDist}) for the distributed approach. 
\item We formulate the camera location estimation problem of SfM in terms of the pairwise line measurements (between the camera locations in $\R^3$). Moreover, we show how to improve the stability of our approach to outlier pairwise measurements by robust preprocessing of the available pairwise camera information and describe a robust iterative approach for camera orientation estimation.
\item We also demonstrate the efficiency and the accuracy of our formulations via synthetic and real data experiments. We note that these experiments show a specifically important characteristic of the SDR formulation: As long as the level of noise in the pairwise lines is below some threshold, the SDR always produces rank-$1$ solutions, i.e. it actually solves the original non-convex program, i.e. the relaxation is tight\footnote{This is sometimes referred to as ``exact rank recovery'', and is not to be confused with our exact ``location'' recovery results for the SDR in the presence of noiseless pairwise lines.}. Also, for higher noise levels, even when the SDR does not produce rank-$1$ solutions, its solution typically has a large spectral gap (i.e., it can be well approximated by a rank-$1$ matrix). In other words, we do not observe a sharp phase transition in the quality of the relaxation.
\item Since the existing results of parallel rigidity theory (cf. Appendix~\ref{Apdx:ParallelRigidity} and, e.g.,~\cite{WhiteleyMatroidBook, WhiteleyMatroid, ErenNetwork, ErenNetwork2, ServatiusWhiteleyCAD}) have not been previously applied to the camera location estimation problem, we provide a summary of the main results of parallel rigidity theory, which completely characterize the conditions for the problem to be well-posed (in $\R^d$, for arbitrary $d$). Also, we formulate a randomized algorithm to efficiently decide in the existence of these conditions.
\end{romannum}

\indent In the literature, convex programming relaxations (and in particular semidefinite programs) have previously served as convex surrogates for non-convex (particularly NP-hard) problems arising in several different areas, such as sensor network localization (from pairwise distances)~\cite{SNLSDP,SNLSDP1}, low-rank matrix completion~\cite{CandesMC1}, phase retrieval~\cite{CandesPR}, robust principal component analysis (PCA)~\cite{CandesRPCA}, multiple-input multiple-output (MIMO) channel detection~\cite{MIMO1,MIMO2}, and many others (also see~\cite{MAXCUTSDP,SSReaper,BoydCVXSurvey,AMCSDP}). Notably, the SDP formulation for sensor network localization~\cite{SNLSDP,SNLSDP1} is not guaranteed (even in the noiseless case) to provide the unique configuration of a globally rigid framework (cf.~\cite{AspnesSurvey}, for global rigidity and other fundamental concepts in ``classical'' rigidity theory). Only if the framework is ``uniquely localizable''~\cite{SNLSDP}, then the SDP is guaranteed to provide the unique solution in the noiseless case. In contrast, our SDR formulation is guaranteed to provide the unique solution (up to global scale, translation and negation) for a parallel rigid framework (cf. \S\ref{sec:ParRigidity}). Similar to our case, the tightness of the relaxation, i.e. obtaining rank-$1$ solutions from semidefinite relaxations, is also observed in several different SDR formulations (see, e.g.,~\cite{KunalRegistration,AfonsoSDRtight} and the survey~\cite{AMCSDP}).\\ \\ 
\noindent {\bf Organization of the Paper:} In \S\ref{sec:TransEstim} we provide the connection to parallel rigidity theory and introduce the SDR formulation. Also, we prove exact (in the noiseless case) and stable (in the noisy case) location recovery results, and formulate an alternating direction method for the SDR. In \S\ref{sec:DistApproach} we introduce the distributed approach and prove its well-posedness. In \S\ref{sec:CamMotEst}, we formulate the camera location estimation problem in SfM as a problem of estimating locations from their pairwise line measurements. We also present the robust camera orientation and pairwise line estimation procedures. We evaluate the experimental performance of our algorithm in \S\ref{sec:Experims}, using synthetic and real data sets. Lastly, \S\ref{sec:Conclusion} is a summary. \\ \\
\noindent {\bf Reproducibility:} The methods and algorithms presented in this paper are packaged in a MATLAB toolbox that is freely available for download from the first author's webpage at \url{http://www.math.princeton.edu/~oozyesil/}.\\ \\
\noindent {\bf Notation:} We denote vectors in $\R^d$, $d\geq 2$, in boldface. $I_d$ and $J_{d}$ are used for the $d\times d$ identity and all-ones matrices, respectively. $S^d$ and $\mbox{SO}(d)$ denote the (Euclidean) sphere in $\R^{d+1}$ and the special orthogonal group of rotations acting on $\R^d$, respectively. We use the hat accent, to denote estimates of our variables, as in $\hat{X}$ is the estimate of $X$. We use star to denote solutions of optimization problems, as in $X^*$. For an $n\times n$ symmetric matrix $A$, $\lambda_1(A) \leq \lambda_2(A) \leq \ldots \leq \lambda_n(A)$ denote its eigenvalues (in ascending order) and $A\succeq0$ denotes that $A$ is positive semidefinite (i.e., $\lambda_i(A)\geq 0$, for all $i = 1,2,\ldots,n$). Also, for an $n\times n$ matrix $X$, $\diag(X)$ denotes the vector formed by the diagonal entries of $X$, and conversely, for $\mathbf{x}\in\R^n$, $\Diag(\mathbf{x})$ denotes the diagonal matrix formed by the entries of $\mathbf{x}$. Lastly, we use the letters $n$ and $m$ to denote the number of locations $|V_t|$ and the number of edges $|E_t|$ of graphs $G_t = (V_t,E_t)$ that encode the pairwise line information.

\section{Location Estimation}
\label{sec:TransEstim}
Consider a graph $G_t = (V_t,E_t)$ of pairwise lines, with each edge $(i,j)\in E_t$ endowed with a pairwise line $\Gamma_{ij}$ corresponding to the locations $\{\mathbf{t}_i\}_{i\in V_t}$ (i.e., satisfying~(\ref{eq:GammaMats})). Given this information, we first address the question of unique realizability of the locations in the next subsection, followed by our SDP formulation for location estimation (from noisy pairwise lines).

\subsection{Parallel Rigidity}
\label{sec:ParRigidity}
The unique realizability of locations (or {\em the solvability problem}) from pairwise line measurements was previously studied, mainly under the name of {\em parallel rigidity theory} (see, e.g.,~\cite{WhiteleyMatroidBook, WhiteleyMatroid, ErenNetwork, ErenNetwork2, ServatiusWhiteleyCAD}). However, to the best of our knowledge, the concepts and the results of parallel rigidity theory have not been related to the well-posedness of the camera location estimation part of SfM, which is why we study them again here. Note that while camera orientation estimation from noiseless pairwise ratios (cf. \S\ref{sec:CamMotEst}) only requires the connectivity of the measurement graph, connectivity alone is insufficient for camera location estimation (see Figure~\ref{fig:ParallelRigidityEx} for such an instance). To address this problem, we now briefly discuss the main results of parallel rigidity theory. For further details on fundamental concepts and results in parallel rigidity theory, see Appendix~\ref{Apdx:ParallelRigidity}.
\begin{figure}[!htbp]
\begin{center}
   \includegraphics[trim=0cm 0cm 0cm 0cm, clip=true, width=0.75\linewidth]{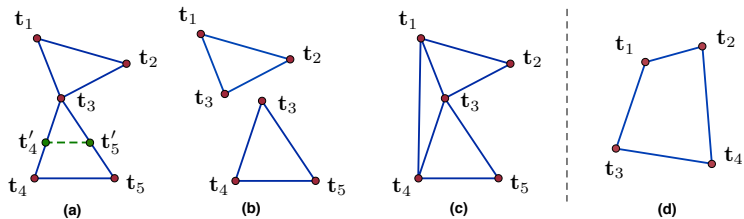}
\end{center}
   \caption{{\bf(a)} A formation of $5$ locations on a connected graph, which is {\em not} parallel rigid (in $\R^2$ and $\R^3$). Non-uniqueness is demonstrated by two non-congruent location solutions $\{\mathbf{t}_1,\mathbf{t}_2,\mathbf{t}_3,\mathbf{t}_4,\mathbf{t}_5\}$ and $\{\mathbf{t}_1,\mathbf{t}_2,\mathbf{t}_3,\mathbf{t}_4',\mathbf{t}_5'\}$, each of which can be obtained from the other by an {\bf independent rescaling} of the solution for one of its maximally parallel rigid components, {\bf(b)} Maximally parallel rigid components of the formation in {\rm(a)}, {\bf(c)} A parallel rigid formation (in $\R^2$ and $\R^3$) obtained from the formation in {\rm(a)} by adding the extra edge $(1,4)$ linking its maximally parallel rigid components, {\bf(d)} A formation of $4$ locations, which is parallel rigid in $\R^3$, but is not parallel rigid in $\R^2$\label{fig:ParallelRigidityEx}}
\end{figure}

 \indent We call the (noiseless) pairwise line information $\{\Gamma_{ij}\}_{(i,j)\in E_t}$ a ``formation'' and consider the following fundamental questions: Using this information, can we {\em uniquely} realize the points $\{\mathbf{t}_i\}_{i\in V_t}$, of course, up to a global translation, scale and negation (i.e. can we realize a set of points {\em congruent} to $\{\mathbf{t}_i\}_{i\in V_t}$)? Is unique realizability a generic property (i.e. is it only a function of the underlying graph, independent of the particular realization of the points, assuming they are in generic position) and can it be decided efficiently? If we cannot uniquely realize $\{\mathbf{t}_i\}_{i\in V_t}$, can we efficiently determine maximal components of $G_t$ that can be uniquely realized? These questions were previously studied in several different contexts like discrete geometry, bearing and angle based sensor network localization and robotics, under various formulations and names (see~\cite{WhiteleyMatroidBook, WhiteleyMatroid, ErenNetwork, ErenNetwork2, ServatiusWhiteleyCAD, KatzUnique,MaximalRigid,AddRef2} and references therein).
\newline \indent The identification of parallel rigid formations is addressed in~\cite{WhiteleyMatroidBook,WhiteleyMatroid,ErenNetwork,ErenNetwork2} (also see~\cite{KatzUnique} and the survey~\cite{JacksonJordanSurvey}), where it is shown that parallel rigidity in $\R^d$ ($d\geq2$) is a generic property of the measurement graph $G_t$ equivalent to unique realizability, and admits a complete combinatorial characterization (a generalization of Laman's condition from classical rigidity to parallel rigidity):
 \begin{theorem}
\label{thm:LamanConds}
For a graph $G = (V,E)$, let $(d - 1)E$ denote the set consisting of $(d - 1)$ copies of each edge in $E$. Then,
$G$ is generically parallel rigid in $\R^d$ if and only if there exists a nonempty $D\subseteq(d - 1)E$, with $|D| = d|V| - (d+1)$, such that for all subsets $D'$ of $D$, we have
\begin{equation}
\label{eq:LamanIneqs}
|D'| \leq d|V(D')| - (d+1) \ ,
\end{equation}
where $V(D')$ denotes the vertex set of the edges in $D'$.
\end{theorem}

\indent To elucidate the conditions of Theorem~\ref{thm:LamanConds}, let us consider the simple examples provided in Figure~\ref{fig:ParallelRigidityEx}: For $\R^2$, if exists, the certificate set $D$ satisfying the conditions in Theorem~\ref{thm:LamanConds} is simply a subset of the edge set $E$. The graphs in the subfigures {\bf (b)} and {\bf (c)} are minimally parallel rigid in $\R^2$, i.e. $D = E$ is the certificate. On the other hand, the graphs in the subfigures {\bf (a)} and {\bf (d)} do not have sufficiently many edges to be parallel rigid, i.e. even if we consider $D$ to be the set of all edges $E$, we get $|D| < 2|V| - 3$. For $\R^3$, let us first consider the (triangle) graphs in the subfigure {\bf (b)}: If we set $D\subseteq2E$ to be the set of two copies of any of the two edges in $E$, and a single copy of the remaining edge (e.g., $D = \{(1,2)_1,(1,2)_2, (2,3)_1,(2,3)_2,(1,3)_1\}$, where $(i,j)_k$ denotes the $k$'th copy of the edge $(i,j)\in E$), then $D$ is a certificate. For the graph in the subfigure {\bf (c)}, $D = \{(1,2)_1,(1,2)_2,(1,3)_1,(1,3)_2,(1,4)_1,(2,3)_1,(3,4)_1,(3,4)_2,(3,5)_1,(4,5)_1,(4,5)_2\}$ satisfies the conditions of Theorem~\ref{thm:LamanConds}. Also, for the graph in the subfigure {\bf (d)}, $D = 2E$ is the certificate. On the other hand, if we consider the graph in the subfigure {\bf (a)}, $D$ can satisfy the first condition $|D| = 11$, if and only if $D = 2E\setminus\{(i,j)_k\}$, for some $i,j\in V$, $k\in{1,2}$. However, in this case, $D$ has a subset $D'$ consisting of two copies of each of the three edges in a triangle graph, which violates the condition (\ref{eq:LamanIneqs}).
\newline\indent The conditions of Theorem~\ref{thm:LamanConds} can be used to design efficient algorithms (e.g., adaptations of the pebble game algorithm~\cite{PebbleGame}, with a time complexity of $\mathcal{O}(n^2)$) for testing parallel rigidity. In Appendix~\ref{Apdx:ParallelRigidity} we detail a randomized algorithm (having time complexity $\mathcal{O}(m)$) for testing parallel rigidity. Moreover, polynomial time algorithms for finding maximally parallel rigid components of non-parallel rigid formations are provided in~\cite{MaximalRigid,KatzUnique}.
\newline\indent In the presence of noiseless pairwise line measurements, parallel rigidity (and hence, unique realizability) is equivalent to error-free location recovery. However, for real images, we are provided with {\em noisy} measurements. In this case, instead of uniqueness of the solution of a specific location estimation algorithm, we consider the following question: Is there enough information for the problem to be {\em well-posed}, i.e. if the noise is small enough, can we estimate the locations stably? For formations which are not parallel rigid, instability may result from independent scaling and translation of maximally rigid components. In this sense, we consider problem instances on parallel rigid graphs to be well-posed. 
\newline \indent Considering the existing results of parallel rigidity theory, we take the following approach for location estimation: Given a (noisy) formation $\{\Gamma_{ij}\}_{(i,j)\in E_t}$ on $G_t = (V_t,E_t)$, we first check for parallel rigidity of $G_t$, if the formation is non-rigid we extract its maximally rigid components (using the algorithm in~\cite{MaximalRigid}) and estimate the locations for the largest maximally rigid component. The details of our formulation for (stable) location estimation are provided in the next section.

\subsection{Location Estimation by SDR}
\label{sec:SDREstim}
In this section we introduce and describe our location estimation algorithm. Suppose that we are given a set of (noisy) pairwise line measurements $\{\Gamma_{ij}\}_{(i,j)\in E_t}$ on $G_t = (V_t,E_t)$. Also, for each $(i,j)\in E_t$, let $Q_{ij}\defeq I_d - \Gamma_{ij}$ denote the matrix of projection onto the $(d-1)$-dimensional subspace orthogonal to this line. Firstly, let us consider the least squares approach studied in~\cite{MicaAmitSfM,BATL2}\footnote{In~\cite{MicaAmitSfM}, the least squares approach is studied specifically for the location estimation problem in $\R^3$ with a slightly different definition for the pairwise measurements $\Gamma_{ij}$'s (also see \S\ref{sec:CamMotEst}), whereas in~\cite{BATL2}, it is studied in arbitrary dimension using {\em unnormalized} $\Gamma_{ij}$'s (i.e., $\Gamma_{ij}$'s do not necessarily satisfy $\tr(\Gamma_{ij}) = 1$).}, which is equivalent to solving the following (non-convex) program
\begin{subequations}
\label{eq:L2TransEst}
\begin{align}
\underset{{\scriptstyle \{\mathbf{t}_i\}_{i\in V_t} \subseteq \R^d}}{\text{minimize}}
& \ \ \sum_{(i,j)\in E_{t}}  \left(\mathbf{t}_i - \mathbf{t}_j\right)^TQ_{ij}\left(\mathbf{t}_i - \mathbf{t}_j\right)\\
\text{subject to} & \ \ \sum_i \mathbf{t}_i = \mathbf{0} \ , \ \sum_i\|\mathbf{t}_i\|_2 ^2= 1\ 
\end{align}
\end{subequations}
Note that we can rewrite the cost function in~(\ref{eq:L2TransEst}) as $\sum_{(i,j)\in E_{t}}  \|Q_{ij}\left(\mathbf{t}_i - \mathbf{t}_j\right)\|_2^2$ (since the projection matrices $Q_{ij}$ satisfy $Q_{ij} = Q_{ij}^TQ_{ij}$), which is why we call~(\ref{eq:L2TransEst}) the {\em least squares} approach. The constraints in~(\ref{eq:L2TransEst}) are designed to exclude the trivial case $\mathbf{t}_i = \mathbf{t}_0, \forall i\in V_t$, for some $\mathbf{t}_0 \in\R^d$. In fact, (\ref{eq:L2TransEst}) is an eigenvalue problem and hence can be efficiently solved. However, for large and noisy data sets, this formulation turns out to be ``ill-conditioned'' in the following sense: The solution has the tendency to ``collapse'', i.e. the problem is not sufficiently constrained to prevent the (less trivial, but still undesired) solutions of the form $\hat{\mathbf{t}}_i \simeq \mathbf{t}_0, \forall i\in V_t\setminus \{i^*\}$ and $\hat{\mathbf{t}}_{i^*} \simeq -\sum_{i\in V_t\setminus \{i^*\}}\hat{\mathbf{t}}_i$, where $i^*$ has a (relatively) small degree in $G_t$ (for such collapsing solutions in $\R^3$, see Figure~\ref{fig:NotreL2Cluster} for a real data set, and Figure~\ref{fig:CollapseExample} for synthetic data). For large data sets having nodes with significantly varying degrees, collapsing solutions of (\ref{eq:L2TransEst}) (sometimes occurring in more than one group of points) can be observed even for low levels of noise in the $Q_{ij}$'s. 
It is worthwhile noting that the problem of collapsing solutions is not caused by the quadratic nature of the cost function in (\ref{eq:L2TransEst}): Formulations of (\ref{eq:L2TransEst}) using sum of (unsquared) projection errors, i.e. $\sum_{(i,j)\in E_{t}}  \|Q_{ij}\left(\mathbf{t}_i - \mathbf{t}_j\right)\|_2$, as the cost function (which are also significantly harder to solve) that we studied using (heuristic) iteratively-reweighted least squares solvers exhibit even worse characteristics in terms of collapsing solutions.
\newline \indent We overcome the collapsing behavior in two steps, first by introducing non-convex ``repulsion constraints'' in (\ref{eq:L2TransEst}) and then formulating an SDR version of the resulting non-convex problem. The non-convex problem is given by
\begin{subequations}
\label{eq:NonConvexL2}
\begin{align}
\underset{{\scriptstyle \{\mathbf{t}_i\}_{i\in V_t} \subseteq \R^d}}{\text{minimize}}
& \ \ \sum_{(i,j)\in E_{t}}  \tr\left(Q_{ij}\left(\mathbf{t}_i - \mathbf{t}_j\right)\left(\mathbf{t}_i - \mathbf{t}_j\right)^T\right)\\
\text{subject to} & \ \  \|\mathbf{t}_i - \mathbf{t}_j\|_2^2 \geq c, \ \forall (i,j) \in E_t \ , \\
& \ \ \sum_i \mathbf{t}_i = \mathbf{0}
\end{align}
\end{subequations}
where $c\in\R^+$ is a constant fixing the (undetermined) scale of the solution ({\em wlog} we take $c = 1$) and the cost function is rewritten in a slightly different way. The repulsion constraints $\|\mathbf{t}_i - \mathbf{t}_j\|_2^2 \geq 1$ are non-convex constraints making the estimation problem difficult even for small-sizes. We introduce a matrix $T$ of size $dn\times dn$, whose $d\times d$ blocks are given by $T_{ij} = \mathbf{t}_i\mathbf{t}_j^T$. Consequently, $T\succeq0$ and $\rank(T) = 1$. To rewrite the cost function in~(\ref{eq:NonConvexL2}) linearly in terms of $T$, we define a Laplacian matrix $L\in\R^{dn\times dn}$, whose $d\times d$ blocks are given by
\begin{equation}
\label{eq:CostLaplacian}
L_{ij} = \begin{cases} &\hspace{0.39in} -Q_{ij} \ \hspace{0.515in} \mbox{for} \ (i,j)\in E_t \\ &\sum_{\{k:(i,k)\in E_t\}} \ Q_{ik} \ \ \hspace{0.115in}\mbox{for} \ i=j \\ & \hspace{0.5in} 0_{d} \ \hspace{0.6in} \mbox{else}\end{cases}
\end{equation}
and which satisfies $\tr(LT) = \sum_{(i,j)\in E_t}\tr(Q_{ij}\left(\mathbf{t}_i - \mathbf{t}_j\right)\left(\mathbf{t}_i - \mathbf{t}_j\right)^T)$. Note that $L$ is symmetric, since $Q_{ij} = Q_{ij}^T$ and $Q_{ij} = Q_{ji}$, and also positive semidefinite, since for $\mathbf{t}\in\R^{dn}$, $\mathbf{t}^TL\mathbf{t} = \sum_{(i,j)\in E_t}\|Q_{ij}(\mathbf{t}_i-\mathbf{t}_j)\|_2^2\geq0$. Also, for every edge $(i,j)\in E_t$, we define a matrix $C^{ij}\in\R^{dn\times dn}$, whose $kl$'th $d\times d$ block is given by
\begin{equation}
\label{eq:ConstLaplacian}
C^{ij}_{kl} = \begin{cases} &\ \ I_d  \ \ \ \  \hspace{0.07in}\mbox{for} \ (k,l) = (i,i) \ \mbox{or} \ (k,l) = (j,j) \\ &-I_d \ \ \ \ \hspace{0.045in}\mbox{for} \ (k,l) = (i,j) \ \mbox{or} \ (k,l) = (j,i) \\ & \ \ 0_{d}  \ \ \ \ \hspace{0.07in}\mbox{else}\end{cases}
\end{equation}
and which allows us to rewrite the inequality constraints $\|\mathbf{t}_i - \mathbf{t}_j\|_2^2 \geq 1$ in~(\ref{eq:NonConvexL2}), linearly in $T$ as $\tr(C^{ij}T) \geq 1$. Moreover, the equality constraint $\sum_i\mathbf{t}_i = \mathbf{0}$ can be rewritten linearly in $T$ as $\tr\left(HT\right) = 0$, for $H = J_{n}\otimes I_d$.
\newline \indent We now relax the only non-convex constraint, that is $\rank(T) = 1$, to formulate the following SDR (known as ``matrix lifting''):
\begin{subequations}
\label{eq:SDRTransEst}
\begin{align}
\underset{{\scriptstyle T \in \R^{dn\times dn}}}{\text{minimize}}
& \ \ \tr\left(LT\right)\\
\text{subject to} & \ \ \tr\left(C^{ij}T\right)\geq 1\ , \ \forall (i,j)\in E_t \ ,\\
& \ \ \tr\left(HT\right) = 0 \ , \\
& \ \ \ T \succeq 0 
\end{align}
\end{subequations}
After solving for $T^*$ in~(\ref{eq:SDRTransEst}), we compute the location estimates $\{\hat{\mathbf{t}}_i\}_{i\in V_t}$
by a deterministic rounding procedure, i.e. we compute the leading eigenvector $\hat{\mathbf{t}}$ of $T^*$ and let $\hat{\mathbf{t}}_i$ be given by the $i$'th $d\times1$ block of $\hat{\mathbf{t}}$.

\subsection{Stability of SDR}
\label{sec:SDRStability}
In this section we analyze the SDR~(\ref{eq:SDRTransEst}) in terms of exact location recovery in the presence of noiseless pairwise line information and stable recovery with noisy line information. \\
\indent We first introduce some notation to simplify our statements. Consider a set of locations $\{\mathbf{t}_i^0\}_{i\in V_t}\subseteq \R^d$ in generic position and let $\gamma_{ij}^0 = \frac{\mathbf{t}_i^0-\mathbf{t}_j^0}{\|\mathbf{t}_i^0-\mathbf{t}_j^0\|_2}$ denote the unit vector from $\mathbf{t}_j^0$ to $\mathbf{t}_i^0$. Then, the (noiseless) formation corresponding to $\{\mathbf{t}_i^0\}_{i\in V_t}$ is given by $\{\Gamma_{ij}^0\}_{(i,j)\in E_t}$, where $\Gamma_{ij}^0 = \gamma_{ij}^0(\gamma_{ij}^0)^T$. We also let $Q_{ij}^0$ denote the projection matrices onto the $(d-1)$-dimensional subspace orthogonal to $\gamma_{ij}^0$, i.e. $Q_{ij}^0 = I_d - \Gamma_{ij}^0$, $L^0$ denote the Laplacian matrix corresponding to $Q_{ij}^0$'s (cf. (\ref{eq:CostLaplacian})), and  $T_0\in\R^{dn\times dn}$ denote the matrix of the noiseless locations, i.e. $T_0 = \mathbf{t}^0(\mathbf{t}^0)^T$ where the i'th $d\times 1$ block of $\mathbf{t}^0$ is equal to $\mathbf{t}_i^0$.\\
\indent We start with the simpler exact recovery result.
\begin{proposition}[Exact Recovery in the Noiseless Case]
\label{propo:ExactRecovery}
Assume that the (noiseless) formation $\{\Gamma_{ij}^0\}_{(i,j)\in E_t}$ is parallel rigid. Then, the SDR~$(\ref{eq:SDRTransEst})$ (with $L = L^0$), followed by the deterministic rounding procedure, recovers the locations exactly, in the sense that any rounded solution of the SDR is congruent to $\{\mathbf{t}_i^0\}_{i\in V_t}$.
\end{proposition}
\begin{proof}
{\em Wlog}, we assume $\min_{(i,j)\in E_t} \|\mathbf{t}_i^0-\mathbf{t}_j^0\|_2 = 1$ and $\sum_i \mathbf{t}_i^0 = \mathbf{0}$. Then we have $\tr(L^0T_0) = 0$, i.e. $T_0$ is a minimizer of (\ref{eq:SDRTransEst}) (since $T_0$ is also feasible by construction). The parallel rigidity of the formation implies $\rank(L^0) = dn - (d+1)$, where the only eigenvector $\mathbf{t}^0$ of $T_0$ with non-zero eigenvalue and the $d$ eigenvectors $\mathbf{v}^1_H,\mathbf{v}^2_H,\ldots, \mathbf{v}^d_H$ of $H$ (in~(\ref{eq:SDRTransEst})) with nonzero eigenvalues form an orthogonal basis for the nullspace of $L^0$ (see, Appendix~\ref{Apdx:ParallelRigidity}). Let $\mathbf{u}_0^i$, $i = d+2,\ldots,dn$, denote the (normalized) eigenvectors of $L^0$ corresponding to its positive eigenvalues. Consider an arbitrary minimizer $T^*$ of (\ref{eq:SDRTransEst}). Since $\tr(L^0T^*) = \sum_{i=d+2}^{dn}  \lambda_i(L^0)(\mathbf{u}_0^i)^TT^*\mathbf{u}_0^i= 0$, where $\lambda_i(L^0) > 0$, $T^*\succeq0$ satisfies $(\mathbf{u}_0^i)^TT^*\mathbf{u}_0^i = 0$ for all $i = d+2,\ldots,dn$. Also, by the feasibility of $T^*$, we get $\tr(HT^*) = \sum_{i=1}^{d} (\mathbf{v}^i_H)^TT^*\mathbf{v}^i_H = 0$, i.e. $(\mathbf{v}^i_H)^TT^*\mathbf{v}^i_H = 0$ for all $i = 1,\ldots,d$. Hence, $\{\mathbf{v}^1_H, \ldots, \mathbf{v}^d_H, \mathbf{u}_0^{d+2}, \ldots, \mathbf{u}_0^{dn}\}$ form an orthogonal basis for the nullspace of $T^*$. This implies $\rank(T^*)=1$, where $T^*$ is of the form $T^* = \alpha T_0$  for some $\alpha \geq 1$ (by the feasibility of $T^*$), establishing the uniqueness of the solution up to scale. As a result, applying the rounding procedure to any solution of~(\ref{eq:SDRTransEst}) yields exact recovery of $\mathbf{t}_i^0$'s (of course, up to congruence).\qquad\end{proof}
\newline\indent Our next result is the stability of the SDR with noisy pairwise line information.
\newline\noindent {\bf Noise Model and Error Function:} We let each edge $(i,j)$ of the measurement graph $G_t = (V_t,E_t)$ be endowed with a line measurement $\Gamma_{ij} = \gamma_{ij}\gamma_{ij}^T$, where $\gamma_{ij} = \gamma_{ij}^0 + \epsilon_{ij}$ with $\|\epsilon_{ij}\|_2 \leq \epsilon$ and $\|\gamma_{ij}\|_2 = 1$. Also, $L_G = D_G-A_G$ denotes the Laplacian of the graph $G_t$, where $D_G$ is the (diagonal) degree matrix of $G_t$ (whose $i$'th diagonal element is equal to the degree of the $i$'th node) and $A_G$ is the adjacency matrix of $G_t$.\\
\indent This time we assume ({\em wlog}), $\tr(T_0) = 1$ and $\tr(HT_0) = 0$. For a solution $T^*$ of the SDR (\ref{eq:SDRTransEst}), we consider the following error function as our measure of stability
\begin{equation}
\label{eq:ErrorMeasure}
\delta(T^*,T_0) = \min_{c\geq 0} \|cT^* - T_0\|_F = \|\frac{\tr(T^*T_0)}{\|T^*\|_F^2}T^*-T_0\|_F
\end{equation}
We are now ready to present the main result of this section:
\begin{theorem}[Stability of SDR Solution]
\label{thm:WeakStability}
Consider a set of noisy pairwise line measurements $\{\Gamma_{ij}\}_{(i,j)\in E_t}$ related to the (noiseless) parallel rigid formation $\{\Gamma_{ij}^0\}_{(i,j)\in E_t}$ as in the noise model given above, and let $T^*$ be a solution of~$(\ref{eq:SDRTransEst})$. Then,
\begin{equation}
\label{eq:WeakStability}
\delta(T^*,T_0) \leq \epsilon\left[\alpha_1 + \left(\alpha_1^2 + 2\alpha_2\right)^{\frac{1}{2}}\right]
\end{equation}
where, the (data dependent) constants $\alpha_1$ and $\alpha_2$ are given by $\alpha_1 = \frac{\sqrt{2}m}{\lambda_{d+2}(L^0)}$ and $\alpha_2 = (\frac{\kappa\sqrt{d}\|L_{G}\|_F}{m} + 1)\frac{\lambda_{n}(L_G)}{\lambda_{d+2}(L^0)}$. Here, the parameter $\kappa$ is given by $\kappa = (\min_{(i,j)\in E_t}\|\mathbf{t}^0_i-\mathbf{t}_j^0\|_2^2)^{-1}$.
\end{theorem}
\newline\indent{\em Proof}. See Appendix~\ref{Apdx:SDRStability}.
\newline\indent We can obtain the stability of the estimated locations, i.e. the rounded solution of~(\ref{eq:SDRTransEst}), as a corollary of Theorem~\ref{thm:WeakStability}:
\begin{corollary}[Location Stability]
\label{corol:LocationStability}
Let $T^*$ be as in Theorem~$\ref{thm:WeakStability}$ and $\hat{\mathbf{t}}$ denote its normalized eigenvector corresponding to its largest eigenvalue. Then
\begin{equation}
\label{eq:LocationStability}
\min_{a\in \R} \|a\hat{\mathbf{t}} - \mathbf{t}^0\|_2 \leq \epsilon \left[\frac{\pi d(n-1)}{2}\left(\alpha_1 + \left(\alpha_1^2 + 2\alpha_2\right)^{\frac{1}{2}}\right)\right]
\end{equation}
\end{corollary}
\begin{proof}
We use a particular implication of the Davis-Kahan theorem (see, e.g., Ch. 7 in~\cite{BhatiaBook}) in order to relate the recovery error in the rounded solution to the error in the solution of the SDR. To this end, observe that 
\begin{equation}
\label{eq:FroVSEuc}
\min_{a\in \R} \|a\hat{\mathbf{t}} - \mathbf{t}^0\|_2 = \|((\mathbf{t}^0)^T\hat{\mathbf{t}})\hat{\mathbf{t}} - \mathbf{t}^0\|_2 = \|(I_{dn} - \mathbf{t}^0(\mathbf{t}^0)^T)\hat{\mathbf{t}}\hat{\mathbf{t}}^T\|_F
\end{equation}
For a given symmetric matrix $A$ and a subset $S$ of the real line, let $P_A(S)$ denote the projection matrix onto the subspace spanned by the eigenvectors of $A$, whose eigenvalues are in $S$. Then Davis-Kahan theorem implies
\begin{equation}
\label{eq:DavisKahanOrig}
\|P_{A}(S_1)P_B(S_2)\|_F \leq \frac{\pi}{2\rho(S_1,S_2)}\|A-B\|_F \ \,
\end{equation}
where $\rho(S_1,S_2) = \min\{|x-y|:x\in S_1, y\in S_2\}$. In our case, if we let $S_0 = \{0\}$ for $T_0$ and $S^* = \{\lambda_{dn}(\tilde{c}T^*)\}$ for $\tilde{c}T^*$, where $\tilde{c} = (\tr(T^*))^{-1}$, we obtain
\begin{equation}
\label{eq:SpecDavisKahan}
\|P_{T_0}(S_0)P_{\tilde{c}T^*}(S^*)\|_F = \|(I_{dn} - \mathbf{t}^0(\mathbf{t}^0)^T)\hat{\mathbf{t}}\hat{\mathbf{t}}^T\|_F \leq \frac{\pi\tr(T^*)}{2\lambda_{dn}(T^*)}\|\tilde{c}T^*-T_0\|_F \ \,
\end{equation}
Here, we use the fact $\frac{\tr(T^*)}{\lambda_{dn}(T^*)} \leq \rank(T^*)$ and the feasibility of $T^*$, i.e. that $\tr(HT^*) = 0$, to get $\frac{\tr(T^*)}{\lambda_{dn}(T^*)} \leq d(n - 1)$ (in fact, we can construct a solution $T^*$ of the SDR satisfying the stronger bound $\rank(T^*) \leq (\sqrt{8(m+1)+1}-1)/2$, see e.g.~\cite{PatakiRankBnd}, however we ignore this slight improvement for simplicity). Also, considering~(\ref{eq:FroVSEuc}) and~(\ref{eq:WeakStaPoly}) from the proof of Theorem~\ref{thm:WeakStability}, we recover the claim of the corollary.
\qquad\end{proof}
\begin{remark} 
{\rm \hspace{0.1in}We note that the bounds in Theorem~\ref{thm:WeakStability} and Corollary~\ref{corol:LocationStability} are quite loose when compared to our experimental observations. Nevertheless, the recovery error is within a constant factor of the noise level $\epsilon$. Also observe that Proposition~\ref{propo:ExactRecovery}, i.e. exact recovery in the noiseless case, is implied by Theorem~\ref{thm:WeakStability} when $\epsilon = 0$.}
\end{remark}
\begin{remark} 
{\rm \hspace{0.1in}The proximity of the solution $T^*$ of~(\ref{eq:SDRTransEst}) to the space of positive semidefinite rank-$1$ matrices can be considered as a measure of the quality of the relaxation. In our experiments with real images and simulated data, we make the following observations: As long as the noise level $\epsilon$ is below some threshold, we always get $\rank(T^*) = 1$, i.e. we can actually solve the non-convex problem efficiently by the SDR~(\ref{eq:SDRTransEst}). For higher levels of noise, $T^*$ is no longer a rank-1 matrix, but it typically has a large spectral gap $(\lambda_{dn}(T^*) - \lambda_{dn-1}(T^*))/\lambda_{dn}(T^*)$. In other words, we do not observe a sharp phase transition in the quality of the SDR, and the relaxation is stable under various noise models. Figure~\ref{fig:SpectralGap} provides an experimental evaluation in $\R^3$ of our observations about the stability of relaxation, using synthetic measurements generated under the noise model~(\ref{eq:NoiseModel}) of \S\ref{sec:SyntExp} (and assuming $p = 0$ in~(\ref{eq:NoiseModel}), i.e. no outlier measurements, for simplicity) for graphs of $n = 50$ nodes and various edge density $\theta = \frac{2m}{n(n-1)}$. We observe that even if the location recovery performance (represented by {\em normalized root mean squared error} (NRMSE) defined in~(\ref{eq:NRMSE})) degrades as the noise level increases, the tightness of the relaxation is preserved up to relatively high noise levels.}
\end{remark}
\begin{figure}[!htbp]
\begin{center}
   \includegraphics[trim=1.4cm 0.5cm 0.5cm 1cm, clip=true, width=0.7\linewidth]{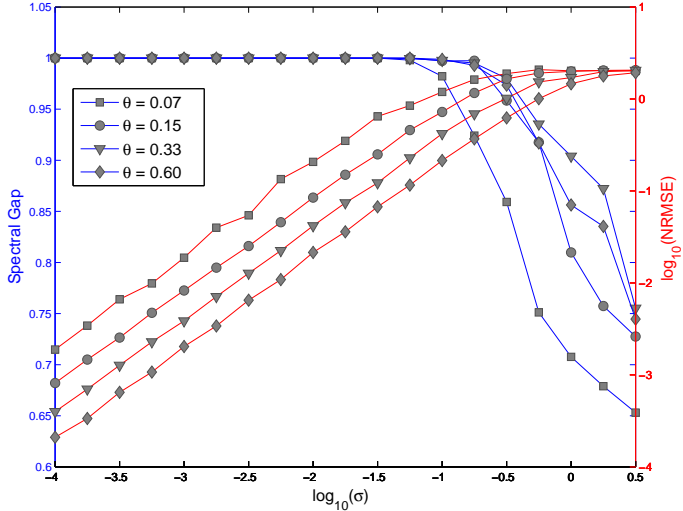}
\end{center}
   \caption{Spectral gap $(\lambda_{3n}(T^*) - \lambda_{3n-1}(T^*))/\lambda_{3n}(T^*)$ of solutions $T^*$ of~{\rm(\ref{eq:SDRTransEst})} and logarithmic recovery error $\log_{10}(\mbox{NRMSE})$ (see~{\rm(\ref{eq:NRMSE})} for NRMSE) versus noise level $\log_{10}(\sigma)$, for graphs with $n = 50$ nodes and various edge density $\theta$ (results are averaged over $10$ random realizations of noisy line measurements and locations)\label{fig:SpectralGap}}
\end{figure}

\subsection{Alternating Direction Augmented Lagrangian Method}
\label{sec:ADM}
The SDR (\ref{eq:SDRTransEst}) is solvable in polynomial time by the classical primal-dual interior-point SDP algorithms (e.g.~\cite{SDPT32}). However, in case of a dense measurement graph $G_t$ (i.e., assuming $m = \mathcal{O}(n^2)$), the interior-point methods become impractical for large number of locations, with a time complexity of $\mathcal{O}(n^6)$ (and a space complexity of $\mathcal{O}(n^4)$) for each iteration of the algorithm. In practice, the computational bottleneck becomes an issue for problem sizes of $n \geq 200$. In this respect, here we provide the details of {\em alternating direction augmented Lagrangian method} (ADM), which is a first-order method with superior computational efficiency~\cite{WenADM}. ADM is an iterative algorithm based on minimization of an augmented Lagrangian function of the dual SDP. In comparison to interior-point methods that aim to satisfy complementary slackness while maintaining primal-dual feasibility at each iteration, ADM aims to construct a primal-dual feasible pair while maintaining complementary slackness. At each iteration, ADM minimizes the dual augmented Lagrangian function first with respect to dual variables, then with respect to dual slack variables and finally updates the primal variables. In the minimization over each variable, ADM regards the other variables as fixed. \\
\indent  In order to obtain an ADM framework for the SDP~(\ref{eq:SDRTransEst}), we rewrite it in standard form and procure the ADM framework (involving variables of larger dimensions) for standard form SDPs developed in~\cite{WenADM}. Such an approach yields a (provably) convergent algorithm, however, in general it has a high computational cost (due to the high dimensionality of the variables associated with the standard form SDP). In our case, we are able to simplify the ADM framework for the standard form of~(\ref{eq:SDRTransEst}) significantly and hence do not artificially increase the computational cost by rewriting~(\ref{eq:SDRTransEst}) in standard form (we also experimentally observed that the ``ad-hoc'' approach in~\cite{WenADM} developed for SDPs involving inequality constraints, which is at least as costly as our resulting algorithm, did not produce a convergent algorithm for~(\ref{eq:SDRTransEst})). We provide the details of rewriting~(\ref{eq:SDRTransEst}) in standard form, constructing the ADM framework for the augmented Lagrangian of its dual and the simplification of the resulting algorithm in Appendix~\ref{Apdx:ADMdetails}. A pseudo-code version of our ADM algorithm is given below (see Appendix~\ref{Apdx:ADMdetails} for the linear operators $\tilde{\mathcal{B}},\tilde{\mathcal{B}}^*$ and the efficient computation of $(\tilde{\mathcal{B}}\tilde{\mathcal{B}}^* + I)^{-1}$).\\
\begin{algorithm}[!htbp]
\caption{Alternating direction augmented Lagrangian method (ADM) for SDP~(\ref{eq:SDRTransEst})\label{alg:ADM}}
\begin{algorithmic}
\STATE Initialize: $T^0\succeq0$ s.t. $\tr(HT^0) = 0$, $R^0\succeq0$, $\mathbf{\nu}^0 \geq \mathbf{0}_{m}$ and $\mathbf{\eta}^0 \geq \mathbf{0}_{m}, \mu > 0$ \vspace{0.1in}
\FOR{$k = 0,1,\ldots$} \vspace{0.05in}
\STATE $\mathbf{z}^{k+1} \ \leftarrow \ -\left(\tilde{\mathcal{B}}\tilde{\mathcal{B}}^* + I\right)^{-1}\left(\frac{1}{\mu}(\tilde{\mathcal{B}}(T^k)-\mathbf{\nu}^k-\mathbf{1}_{m}) + \tilde{\mathcal{B}}(R^k-L)-\mathbf{\eta}^k\right)$ \vspace{0.05in}
\STATE $F^{k+1} \ \leftarrow \ L - \frac{1}{\mu}T^k-\tilde{\mathcal{B}}^*(\mathbf{z}^{k+1})$
\STATE \hspace{0.15in} $\left\lfloor \begin{aligned} &\mbox{Compute the spectral decomposition of} \ F^{k+1}: \\
& F^{k+1} = \left[\begin{smallmatrix}V_+ & V_-\end{smallmatrix}\right]\left[\begin{smallmatrix}D_+ & 0 \\ 0 & D_-\end{smallmatrix}\right]\left[\begin{smallmatrix}V_+^T \\ V_-^T\end{smallmatrix}\right], \ \mbox{where} \ \diag(D_+) > \mathbf{0} \end{aligned} \right.$ \vspace{0.05in}
\STATE $R^{k+1} \ \leftarrow \ V_+D_+V_+^T$ \vspace{0.05in}
\STATE $\mathbf{\eta}^{k+1} \ \leftarrow \ \max\left\{\mathbf{z}^{k+1} - \frac{1}{\mu}\mathbf{\nu}^k,\mathbf{0}_{m}\right\}$ \vspace{0.05in}
\STATE $T^{k+1} \ \leftarrow \ -\mu V_-D_-V_-^T$  \vspace{0.05in}
\STATE $\mathbf{\nu}^{k+1} \ \leftarrow \ -\mu \min\left\{\mathbf{z}^{k+1} - \frac{1}{\mu}\mathbf{\nu}^k,\mathbf{0}_{m}\right\}$ \vspace{0.05in}
\ENDFOR
\end{algorithmic}
\end{algorithm}
\indent We refer the reader to~\cite{WenADM} for practical details related to termination rules using measures of infeasibility, stagnation detection, updating the parameter $\mu$ for faster convergence, additional step size parameter used to update the primal variables $T^k$ and $\nu^k$, and also for convergence analysis of ADM. Considering the convergence rate analysis of ADM provided in~\cite{ADMrate}, we need $\mathcal{O}(1/\epsilon)$ iterations in order to achieve an $\epsilon$ accuracy. Note that, at each iteration, the most computationally expensive step of Algorithm~\ref{alg:ADM} is the spectral decomposition of $F^{k+1}$. However, since we experimentally observe a stable SDP relaxation resulting in a low-rank primal solution $T^*$, computation of $V_-$ and $D_-$ can be greatly simplified by computing only a few negative eigenvalues of $F^{k+1}$ (e.g., by using Arnoldi iterations~\cite{ArnoldiIters}). As a result, assuming $\mathcal{O}(n^3)$ complexity for spectral decomposition, the time complexity of $\mathcal{O}(n^3/\epsilon)$ (already significantly less compared to interior point methods) can be even further reduced.

\section{Distributed Approach}
\label{sec:DistApproach}
The ADM framework introduced in \S\ref{sec:ADM} provides a computationally feasible alternative to classical SDP solvers and allows us to solve the SDR~(\ref{eq:SDRTransEst}) beyond $n \simeq 200$. However, for large sets of images ($n\gg1000$), the need for a distributed algorithm is apparent. In this section, we provide the details of a distributed algorithm for translation estimation, based on spectral graph partitioning and  convex programming.\\
\indent The main structure of our distributed location estimation algorithm is the following: Given a maximum problem size, i.e. an integer $N_{max}$ denoting the maximum number of locations our computational resources can efficiently estimate by~(\ref{eq:SDRTransEst}), we first partition $V_t$ into subsets (that we call ``patches'') of sizes at most $N_{max}$ (by maintaining sufficient connectivity in the induced subgraphs and sufficient overlaps between patches). Then, for each induced subgraph, we extract the maximally parallel rigid components. We then find for each rigid component the ``local'' coordinate estimates by the SDR~(\ref{eq:SDRTransEst}).
Finally, we stitch the local estimates into a global solution by convex programming.\\
\indent We note that, the main idea of our distributed approach, i.e. division of the problem into smaller subproblems and then constructing the global solution from the local solutions, is also adapted for various problems in the literature (see, e.g.,~\cite{AmitMihailSensors}). However, depending on the structure and the challenges of the specific problem studied, these distributed methods usually have significant differences. For instance, as compared to~\cite{AmitMihailSensors}, while the same algorithm (namely {\em the eigenvector method} (EVM)) is used in our approach to remove the pairwise sign ambiguities between local solutions (cf. \S\ref{sec:PatchReg}), the steps involving the graph partitioning and extraction of well-posed local problems, computation of local solutions, and estimation of global locations from (sign corrected) local estimates are significantly different. 

\subsection{Graph Partitioning}
\label{sec:GraphPartition}
In order to partition $V_t$ into sufficiently overlapping subsets (for high quality global reconstruction) with sufficiently dense induced graphs (for high quality local estimation) of sizes bounded by $N_{max}$, we use the following algorithm, which bears a resemblance with the graph partitioning algorithm of \cite{GraphDISCO}. Starting with $\mathcal{G}_t^1 = \{G_t\}$, at the $k$'th iteration partition each graph in $\mathcal{G}_t^k$ (where, $\mathcal{G}_t^k$ denotes the set of graphs to be partitioned) into two subgraphs using the spectral clustering algorithm of~\cite{SpectCluster}. Then, extend the computed partitions to include the $1$-hop neighborhoods of their vertices in $G_t$ (and, of course, the induced edges). Assign the (extended) partitions with sizes smaller than $N_{max}$ to the set of patches, and those with sizes larger than $N_{max}$ to $\mathcal{G}_t^{k+1}$. Repeat until there is no subgraph left to partition, i.e. until the $K$'th iteration, where $\mathcal{G}_t^{K+1} = \emptyset$.
\newline \indent After the partitioning step, we extract the maximally parallel rigid components of each patch as described in \S\ref{sec:ParRigidity} (after this stage we use the term patch for parallel rigid patches). We then remove the patches that are subsets of another patch from the set of patches. We also remove the patches that do not have sufficient overlap (i.e. overlap size $\geq 2$, also see next section) with any other patch (which happens very rarely and is required since they cannot be used in the global location estimation). At this stage, we get a patch graph $G_P = (V_P,E_P)$, where $V_P$ denotes the patches and $(i,j)\in E_P$ if and only if there is sufficient overlap between the patches $P_i$ and $P_j$. Here, if $G_P$ is not connected (which was never the case in our experiments), we can either extract the largest connected component of $G_P$ or extend the patches to include their $1$-hop neighborhoods until $G_P$ is connected for the next steps of our algorithm. We then compute the ``local'' coordinate estimates for these rigid patches (whose negation signs, scales and translations with respect to the global coordinate system are undetermined at this stage) by solving the SDR~(\ref{eq:SDRTransEst}). The computation of the local coordinates for each patch can be done in parallel in a multi-core processing environment, where each processing unit computes the local coordinates of one or more patches.

\subsection{Pairwise Patch Registration}
\label{sec:PatchReg}
After solving the SDR~(\ref{eq:SDRTransEst}) for each patch $P_i$, we obtain estimates $\{\hat{\mathbf{t}}_k^i\}_{k\in P_i}$ of the representations $\{\mathbf{t}_k^i\}_{k\in P_i}$ of the locations in the coordinate system of each patch. The representations $\{\mathbf{t}_k^i\}_{k\in P_i}$ satisfy 
\begin{equation}
\label{eq:LocalCoorSysRepres}
\mathbf{t}_k = c^i\mathbf{t}_k^i + \mathbf{t}^i \ , \ \ k\in P_i\ , i\in V_P
\end{equation}
where $\mathbf{t}_k$ denotes the global coordinates of the $k$'th location, and $c^i$ and $\mathbf{t}^i$ denote the {\em signed} scale and translation of patch $P_i$, respectively (we use the signed scale, i.e. $c^i\in\R$ can assume negative values, because of the unknown negation). Given $\{\hat{\mathbf{t}}_k^i\}_{k\in P_i}$, $i\in V_P$, our objective is to estimate the locations $\{\mathbf{t}_k\}_{k\in\bigcup P_i}$ by formulating an efficient algorithm, which will be robust to possible outliers in the estimates $\hat{\mathbf{t}}_k^i$. In this respect, firstly observe that any algorithm designed to minimize the errors in the linear equations~(\ref{eq:LocalCoorSysRepres}) should also exclude trivial solutions of the form $c^i = 0$ and $\mathbf{t}_k = \mathbf{t}^i = \mathbf{t}$ (for some $\mathbf{t}\in\R^d$), for all $i$ and $k$. However, similar to the location estimation from noisy pairwise lines problem, the existence of this null-space (spanned by the trivial solutions) results in collapsing solutions for under-constrained optimization programs. As in the case of the least squares solver for the location estimation problem, we experimentally observed such collapsing solutions for the spectral method designed to minimize the sum of squared $\ell_2$ norms of the errors in equations~(\ref{eq:LocalCoorSysRepres}) by excluding the solutions in the null-space.
\newline\indent Collapsing solutions can be avoided simply by requiring $|c^i|^2 \geq 1$, for all $i \in V_P$, which is a non-convex constraint. Similar to the construction of the SDR~(\ref{eq:SDRTransEst}), the non-convexity (resulting from the unknown patch signs allowing $c^i$ to assume negative values) can be resolved by using matrix lifting. An efficient method in this direction is the adaptation of the partial matrix lifting (only applied to the variables involved in non-convex constraints) method of~\cite{KunalRegistration} to our problem. In this method, using the sum of squared $\ell_2$ norms of the errors in equations~(\ref{eq:LocalCoorSysRepres}) as the cost function, the unconstrained variables ($\mathbf{t}_k$'s and $\mathbf{t}^i$'s) are analytically computed as functions of the constrained variables (in our case, $c^i$'s) and the resulting quadratic form (in $c^i$'s) is used to define a matrix lifting relaxation for the constrained variables (see~\cite{KunalRegistration} for further details). However, this time, instead of using a matrix lifting method, we pursue a different approach: To overcome the non-convexity in $|c^i|^2\geq 1$, we first estimate the unknown sign of each patch $P_i$ and then impose the convex constraints $c^i \geq 1$ for the sign-corrected patches (i.e. patches with the local estimates $\hat{\mathbf{t}}_k^i$ replaced with $\hat{z}^i\hat{\mathbf{t}}_k^i$, where $\hat{z}^i \in \{-1,+1\}$ is the estimate of the negation $z^i = \mbox{sign}(c^i)$ of patch $P_i$). Estimation of patch signs from pairwise signs (see~(\ref{eq:PatchReg}) for pairwise sign estimation) is performed using the eigenvector method (EVM) (see, e.g., \cite{AmitMihailSensors}), which is a robust and efficient spectral algorithm allowing a reliable estimation of patch signs. Using the estimated signs, we can minimize the sum of {\em unsquared} $\ell_2$ norms in equations~(\ref{eq:LocalCoorSysRepres}) (which cannot be used as a convex cost function in the matrix lifting approach), and hence maintain robustness to outliers in the estimates $\hat{\mathbf{t}}_k^i$. In our experiments with simulated data and real images, this two step formulation produced more accurate location estimates compared to the matrix lifting alternative (with similar running times, since the partial matrix lifting results in a semidefinite program with a matrix variable of size $|V_P|\times|V_P|$), making it our choice for stitching the global solution. We now provide the details of the sign estimation procedure, whereas the final step of location estimation from sign-corrected patches is covered in \S\ref{sec:GlobalStitch}. 
\newline\indent In order to estimate the patch signs $\{z^i\}_{i\in V_P}$, the relative pairwise signs $z^{ij} = z^iz^j$, $(i,j)\in E_P$, are estimated first. This is accomplished by solving the following least squares problem for each $(i,j)\in E_P$ 
\begin{equation}
\label{eq:PatchReg}
\underset{{\scriptstyle c^{ij} \in\R, \mathbf{t}^{ij}\in\R^d}}{\text{minimize}}
\ \ \sum_{k\in P_i\cap P_j}\left\| \hat{\mathbf{t}}_k^i - \left(c^{ij}\hat{\mathbf{t}}_k^j + \mathbf{t}^{ij}\right)\right\|_2^2
\end{equation}
where $c^{ij}, \mathbf{t}^{ij}$ denote the relative (signed) scale and translation between $P_i$ and $P_j$, respectively. The relative sign estimate $\hat{z}^{ij} \in \{-1,+1\}$ is given by $\hat{z}^{ij} = \mbox{sign}((c^{ij})^*)$.
\newline\indent Using the relative sign estimates $\{\hat{z}^{ij}\}_{(i,j)\in E_P}$, the sign estimates $\{\hat{z}^i\}_{i\in V_P}$ are computed by EVM, which is a spectral method for finding signs with the goal of satisfying as many equations of the form $\hat{z}^i\hat{z}^j = \hat{z}^{ij}$ for $(i,j)\in E_P$ as possible (see \cite{AmitMihailSensors} for the details). Here, we note that, although the sum of squared norms cost function in~(\ref{eq:PatchReg}) can be replaced by the sum of (unsquared) norms cost to improve robustness to outliers in $\hat{\mathbf{t}}_k^i$'s, we prefer the more efficient least squares version (in fact, (\ref{eq:PatchReg}) has a closed-form solution) since we did not experimentally observe a significant improvement in the accuracy of signs estimated by EVM. 

\subsection{Global Stitching of Local Patches}
\label{sec:GlobalStitch}
Stitching the local patches into a globally consistent $d$-dimensional map comprises the last step of our distributed approach. As we discussed in \S\ref{sec:PatchReg}, we aim to efficiently estimate the global locations $\mathbf{t}_k$ using the linear equations~(\ref{eq:LocalCoorSysRepres}), while maintaining robustness to outliers in $\hat{\mathbf{t}}_k^i$'s and preventing collapsing solutions. In this respect, using the estimated patch signs (i.e., estimates of signs of $c^i$) in~(\ref{eq:LocalCoorSysRepres}), we maintain robustness by minimizing sum of (unsquared) norms of errors in equations~(\ref{eq:LocalCoorSysRepres}), while simply constraining $c^i\geq 1$ to prevent collapse. Hence, we solve the following convex program (using, e.g.~\cite{SDPT32}), in which we jointly estimate the scales $\{c^i\}_{i\in V_P}$ and translations $\{\mathbf{t}^i\}_{i\in V_P}$ of the sign-corrected patches (i.e., patches with the local estimates $\hat{\mathbf{t}}_k^i$ replaced with $\hat{z}^i\hat{\mathbf{t}}_k^i$) and the global locations $\{\mathbf{t}_{k}\}_{k\in \bigcup P_i}$ 
\begin{subequations}
\label{eq:JointlyScTrSynch}
\begin{align}
\underset{{\scriptstyle \left\{ \mathbf{t}_k, \ c^i,\ \mathbf{t}^i\right\}}}{\text{minimize}}
& \ \ \sum_{i\in V_P}\sum_{k\in P_i}\left\| \mathbf{t}_k - \left(c^i\hat{z}^i\hat{\mathbf{t}}_k^i + \mathbf{t}^i\right)\right\|_2\\
\text{subject to} & \ \ c^i \geq 1 , \ \forall i\in V_P 
\end{align}
\end{subequations}

\subsection{Well-posedness of the Distributed Problem}
Similar to the well-posedness of location estimation from pairwise lines, we consider the following question for the distributed problem: Do the local coordinate estimates $\{\hat{\mathbf{t}}_k^i\}_{k\in P_i, \ i\in V_P}$ provide enough information to yield a well-posed instance of the global location estimation problem? Once again, we consider an instance of the distributed problem to be well-posed if the global locations can be uniquely (of course, up to congruence) estimated from the noiseless local coordinates $\{\mathbf{t}_k^i\}_{k\in P_i, \ i\in V_P}$. We (partially) answer this question in Proposition~{\ref{propo:ExactRecoveryDist}}, where it is shown that the local coordinate estimates provided via the specific construction of the patch graph $G_P$ given in \S\ref{sec:GraphPartition} are sufficient for well-posedness. This result is established by proving exact recovery of global locations from noiseless local coordinates using the two step global location construction algorithm.
\begin{proposition}[Exact Recovery from Noiseless Local Coordinates]
\label{propo:ExactRecoveryDist}
Consider a graph $G_P = (V_P,E_P)$ of patches $\{P_i\}_{i\in V_P}$ and a set of (noiseless) local coordinates $\{\mathbf{t}_k^i\}_{k\in P_i, i\in V_P}$ corresponding to the global locations $\{\mathbf{t}_k\}_{k\in \bigcup P_i}$ (i.e., $\mathbf{t}_k^i$ satisfy~$(\ref{eq:LocalCoorSysRepres})$ for a set of signed scales $c^i$ and translations $\mathbf{t}^i$ of the patches, for all $k\in P_i$ and $i\in V_P$). Then, if $G_P$ is connected and satisfies $(i,j)\in E_P$ if and only if $|P_i\cap P_j| \geq2$, the two step global location construction algorithm (i.e., estimation of patch signs by $(\ref{eq:PatchReg})$ and EVM~{\rm\cite{AmitMihailSensors}} followed by global location estimation by $(\ref{eq:JointlyScTrSynch})$) recovers the global locations exactly when provided with the noiseless local coordinates (i.e., $\hat{\mathbf{t}}_k^i = \mathbf{t}_k^i$ for $(\ref{eq:PatchReg})$ and $(\ref{eq:JointlyScTrSynch})$), in the sense that any solution of the algorithm is congruent to $\{\mathbf{t}_k\}_{k\in \bigcup P_i}$.
\end{proposition}
\newline\indent{\em Proof}. See Appendix~\ref{Apdx:ExactRecoveryDist}.
\begin{remark} 
{\rm \hspace{0.1in}We note that, in the presence of noiseless pairwise lines $\{\Gamma_{ij}\}_{(i,j)\in E_t}$ (on parallel rigid $G_t = (V_t,E_t)$), and assuming that we can obtain a connected patch graph $G_P$ from $G_t$ using the graph partitioning procedure of \S\ref{sec:GraphPartition}, Propositions~\ref{propo:ExactRecovery} and~\ref{propo:ExactRecoveryDist} imply exact recovery of $\{\mathbf{t}_k\}_{k\in\bigcup P_i}$ from the lines $\{\Gamma_{ij}\}_{(i,j)\in E_t}$.}
\end{remark}
\begin{remark} 
{\rm \hspace{0.1in}The conditions imposed on $G_P$ in Proposition~\ref{propo:ExactRecoveryDist} (i.e. connectivity and that, for all $(i,j)\in E_P$, $|P_i\cap P_j|\geq2$) are usually not difficult to satisfy. Also, observe that these conditions are independent of the dimension $d$ of the locations (which is not the case, e.g., for the combinatorial conditions in~\cite{LaterRef1,LaterRef2}). However, it may be possible to assume even weaker conditions on $G_P$ to obtain exact recovery results using, e.g., the (partial) matrix lifting method discussed in \S\ref{sec:PatchReg}: We conjecture that if the patches $\{P_i\}_{i\in V_P}$ satisfy a specific $2$-lateration\footnote{We note that the notion of lateration considered here should not be confused with the classical laterated graphs, although it resembles the classical concept in some aspects.} condition, i.e. if there exists a reordering of the patch indices such that, for every $2 \leq i \leq |V_P|$, $P_i$ and $P_1 \cup \ldots \cup P_{i-1}$ have at least $2$ points in common (which is, obviously, a weaker condition compared to the conditions in Proposition~\ref{propo:ExactRecoveryDist}), then the matrix lifting method should recover the locations exactly. On the other hand, since in our two step method, the first step requires the estimation of pairwise signs (in order to estimate the patch signs), the condition that $|P_i\cap P_j|\geq2$, $\forall (i,j)\in E_P$, is in fact necessary for exact recovery.}
\end{remark}

\section{Camera Motion Estimation}
\label{sec:CamMotEst}
In this section, we provide the details of the application of our location estimation algorithm (developed for the general problem in $d$ dimensions) to the camera location estimation part of the structure from motion (SfM) problem in computer vision (defined in $\R^3$). In the SfM problem (see Figure~\ref{fig:SfMProblem}), camera motion estimation is based on point correspondences between pairs of images. As a result, misidentification of corresponding points can induce estimation errors in pairwise line estimates, denoted $\hat{\Gamma}_{ij}$'s, and manifest itself through large errors in the estimated camera locations. In that respect, our primary goal in this section is to formulate a robust (to misidentified corresponding points) and efficient procedure for pairwise line estimation that would then be used as input to our SDR framework (see Figure~\ref{fig:NoiseTol} for a comparison of the accuracy of the line estimates computed using our robust procedure and a simpler estimator, which is not robust to misidentified corresponding points). We also devise a robust algorithm for the camera orientation estimation part, which directly affects the recovery performance of the location estimation part. We start with a brief description of the measurement process.   
\begin{figure}[!htbp]
\begin{center}
   \includegraphics[trim=1cm 1cm 1cm 1cm, clip=true, width=0.7\linewidth]{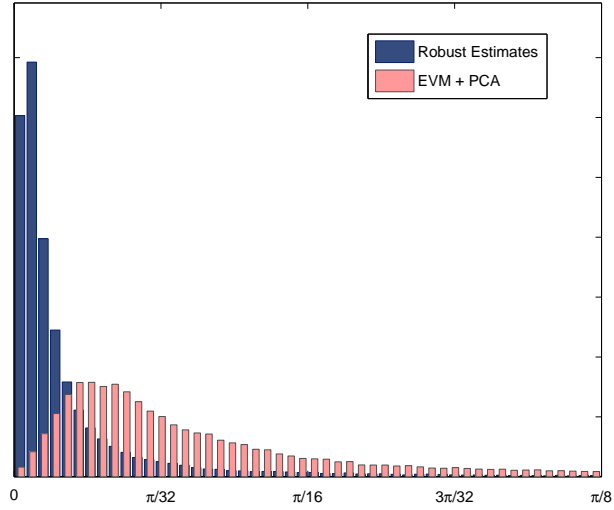}
\end{center}
   \caption{Histogram plots of the errors in line estimates computed by our robust method (cf. {\rm\S\ref{sec:RotEstim}} and {\rm\S\ref{sec:SubspaceEstim}}) and a simpler estimator. The simpler estimator uses the eigenvector method (EVM)~{\rm\cite{MicaAmitSfM}} for rotation estimation, and PCA for subspace estimation (using the noisy estimates of the $2$D subspace samples $\nu_{ij}^k$ in {\rm(\ref{eq:MeasEpiCst})}). The errors represent the angles between the estimated lines and the corresponding ground truth lines (computed from the camera location estimates of~{\rm\cite{SnavelyData}} for the Notre-Dame dataset, studied in {\rm\S\ref{sec:RealDataExp}}). The errors take values in $[0,\pi/2]$, however the histograms are restricted to $[0,\pi/8]$ to emphasize the difference of the quality in the estimated lines. We note that, for the robust method, the percentage of the line estimates having errors larger than $\pi/8$ is $3.7\%$, whereas, for the simple estimator, it is $11.5\%$. \label{fig:NoiseTol}}
\end{figure}
\newline\indent Let $\{\mbox{I}_1,\mbox{I}_2,\ldots,\mbox{I}_n\}$ represent a collection of images of a stationary $3$D scene. We use a pinhole camera model (see Figure~\ref{fig:EpipolarGeo}), and denote the orientations, locations of the focal points, and focal lengths of the $n$ cameras corresponding to these images by $\{R_i\}_{i=1}^n \subseteq \mbox{SO}(3)$, $\{\mathbf{t}_i\}_{i=1}^n \subseteq \R^3$, and $\{f_i\}_{i=1}^n \subseteq \R^+$, respectively. Consider a scene point $\mathbf{P} \in \R^3$ represented in the $i$'th image plane by $\mathbf{p}_i \in \R^3$ (as in Figure~\ref{fig:EpipolarGeo}). To produce $\mathbf {p}_i$, we first represent $\mathbf {P}$ in the $i$'th camera's coordinate system, that is, we compute $\mathbf {P}_i = R_i^T(\mathbf{P} - \mathbf{t}_i) = (\mathbf{P}_i^x,\mathbf{P}_i^y,\mathbf{P}_i^z)^T$ and then project it to the $i$'th image plane by $\mathbf{p}_i = (f_i/P_i^z)\mathbf{P}_i$. Note that, for the image $\mbox{I}_i$, we in fact observe $\mathbf{q}_i = (\mathbf{p}_i^x,\mathbf{p}_i^y)^T \in \R^2$ (i.e., the coordinates on the image plane) as the measurement corresponding to $\mathbf{P}$.
\begin{figure}[!htbp]
\begin{center}
   \includegraphics[trim=0cm 0cm 0cm 0cm, clip=true, width=0.6\linewidth]{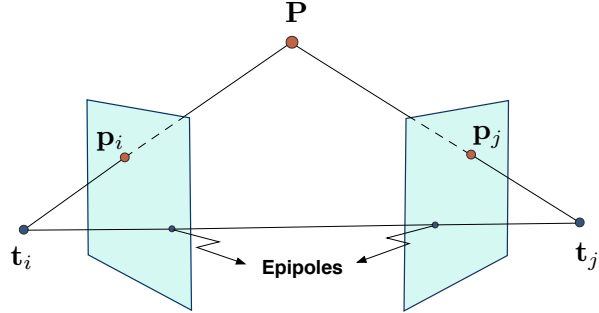}
\end{center}
   \caption{$3$D projective geometry of the pinhole camera model (using virtual image planes for mathematical simplicity)\label{fig:EpipolarGeo}}
\end{figure}
\newline \indent Following the conventions from the SfM literature, for an image pair $\mbox{I}_i$ and $\mbox{I}_j$, the pairwise rotation and translation between the $i$'th and $j$'th camera coordinate frames are denoted by $R_{ij} = R_i^TR_j$ and $\mathbf{t}_{ij} = R_i^T(\mathbf{t}_j-\mathbf{t}_i)$ (not to be confused with $\mathbf{t}_i-\mathbf{t}_j$ or $\gamma_{ij}, \mathbf{t}^{ij}$ used previously), respectively. The essential matrix $E_{ij}$ is then defined by $E_{ij} = [\mathbf{t}_{ij}]_{\times}R_{ij}$, where $[\mathbf{t}_{ij}]_{\times}$ is the skew-symmetric matrix corresponding to the cross product with $\mathbf{t}_{ij}$. If the projections of a $3$D scene point $\mathbf{P}$ onto the $i$'th and the $j$'th image planes are denoted by $\mathbf{p}_i\in\R^3$ and $\mathbf{p}_j\in\R^3$, respectively, the essential matrix $E_{ij}$ satisfies the ``epipolar constraint'' given by 
\begin{equation}
\label{eq:EpipolarConst}
\mathbf{p}_i^TE_{ij}\mathbf{p}_j = 0
\end{equation}
In fact, the epipolar constraint~(\ref{eq:EpipolarConst}) is a restatement of the coplanarity of the three vectors $\mathbf{P}-\mathbf{t}_i$, $\mathbf{P}-\mathbf{t}_j$ and $\mathbf{t}_i-\mathbf{t}_j$ (see Figure~\ref{fig:EpipolarGeo}). However, since~(\ref{eq:EpipolarConst}) is given in terms of the measurable variables $\mathbf{p}_i$, it is used as an extra constraint (on the special structure of $E_{ij}$ having $6$ degrees of freedom) in the estimation of $E_{ij}$. 
\newline\indent Provided with the image set $\{\mbox{I}_i\}_{i=1}^n$, to estimate the essential matrices, we first extract feature points and find pairs of corresponding points between images (see Figure~\ref{fig:CorrPtsEx} for an image pair with corresponding points) using SIFT~\cite{MicaAmit22}, and then estimate the essential matrices using the eight-point algorithm\footnote{We note that, the essential matrix $E_{ij}$ can be estimated using only $5$ point correspondences between the images $\mbox{I}_i$ and $\mbox{I}_j$ (while assuming $\|\mathbf{t}_i-\mathbf{t}_j\|_2=1$). In practice, however, we use the eight-point algorithm, which is an efficient linear method requiring a minimum of $8$ point correspondences.} (see, e.g.,~\cite{HartleyBook}), while also employing the RANSAC protocol (to reduce the effect of outliers in point correspondences). For image pairs with sufficiently many inliers, the estimated essential matrices $\hat{E}_{ij}$ are then (uniquely) factorized into $\hat{R}_{ij}$ and $[\hat{\mathbf{t}}_{ij}]_{\times}$. 
\begin{figure}[!htbp]
\begin{center}
   \includegraphics[trim=0cm 0cm 0cm 0cm, clip=true, width=0.75\linewidth]{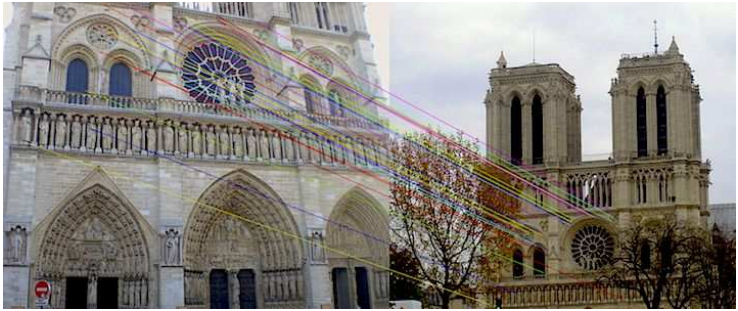}
\end{center}
   \caption{Two images of the Notre-Dame Cathedral set from~{\rm\cite{SnavelyData}}, with corresponding feature points (extracted using SIFT~{\rm\cite{MicaAmit22}}). The essential matrix $E_{ij}$ can be estimated using the (non-linear) five-point or the (linear) eight-point algorithms (see, e.g.,~{\rm\cite{HartleyBook}}), that require at least $5$ and $8$ pairs of corresponding feature points, respectively.\label{fig:CorrPtsEx}}
\end{figure}
\newline\indent Classically, the relative rotation and translation estimates, $\hat{R}_{ij}$ and $\hat{\mathbf{t}}_{ij}$, computed from the decomposition of $\hat{E}_{ij}$, are used to estimate the camera locations $\mathbf{t}_i$. However, for large and unordered collections of images, these estimates usually have errors resulting from misidentified and/or small number of corresponding points. Especially, the erroneous estimates $\hat{\mathbf{t}}_{ij}$ result in large errors for the location estimation part. As a result, instead of using the existing algorithms (e.g.,~\cite{HartleyRotations, MicaAmitSfM, MartinecRotations}) to find the rotation estimates $\hat{R}_i$ and then computing the pairwise line estimates $\hat{\Gamma}_{ij} =  (\hat{R}_i\hat{\mathbf{t}}_{ij})(\hat{R}_i\hat{\mathbf{t}}_{ij})^T$ (assuming $\|\hat{\mathbf{t}}_{ij}\|_2 = 1$) for the SDR solver~(\ref{eq:SDRTransEst}), we follow a different procedure: First, the rotation estimates $\hat{R}_i$ are computed using an iterative, robust algorithm (as detailed in \S\ref{sec:RotEstim}), and then we go back to the epipolar constraints~(\ref{eq:EpipolarConst}) (as explained below) to estimate the pairwise lines using a robust subspace estimation algorithm (cf. \S\ref{sec:SubspaceEstim}).
\newline\indent To clarify the main idea of our robust pairwise line estimation method using the epipolar constraints, we first emphasize the linearity of~(\ref{eq:EpipolarConst}) in the camera locations $\mathbf{t}_i$ and $\mathbf{t}_j$, by rewriting it as (also see~\cite{MicaAmitSfM})
\begin{eqnarray}
\nonumber
\mathbf{p}_i^TE_{ij}\mathbf{p}_j &=&  \mathbf{p}_i^T\left[R_i^T(\mathbf{t}_{j}-\mathbf{t}_{i})\right]_{\times}R_{i}^TR_j\mathbf{p}_j\\
\nonumber
&=& \mathbf{p}_i^TR_i^T\left\{(\mathbf{t}_{j}-\mathbf{t}_{i})\times R_j\mathbf{p}_j\right\} \\
\label{eq:EpipolarRewrite1}
 &=& \left(R_i\mathbf{p}_i\times R_j\mathbf{p}_j \right)^T\left(\mathbf{t}_i -\mathbf{t}_j\right) = 0
\end{eqnarray}
As mentioned before, for an image $\mbox{I}_i$, the measurement corresponding to a $3$D point $\mathbf{P}$ is given in terms of the coordinates of the $2$D image plane by $\mathbf{q}_i = (\mathbf{p}_i^x,\mathbf{p}_i^y)^T \in \R^2$. For an image pair $\mbox{I}_i$ and $\mbox{I}_j$, let $\{\mathbf{q}_{i}^{k}\}_{k=1}^{m_{ij}} \subseteq \R^2$ and $\{\mathbf{q}_{j}^{k}\}_{k=1}^{m_{ij}}\subseteq \R^2$ denote $m_{ij}$ corresponding feature points. Then, using~(\ref{eq:EpipolarRewrite1}), we get (in the noiseless case)
\begin{subequations}
\label{eq:MeasEpiCst}
\begin{align}
&(\mathbf{\nu}_{ij}^k)^T\left(\mathbf{t}_i - \mathbf{t}_j\right) = 0, \ k = 1,\ldots, m_{ij}, \ \ \mbox{for} \\
& \mathbf{\nu}_{ij}^k \defeq \mathbf{\Theta} \left[\left(R_i\left[\begin{matrix} \mathbf{q}_i^k/f_i \\ 1\end{matrix}\right]\right) \times \left(R_j\left[\begin{matrix} \mathbf{q}_j^k/f_j \\ 1\end{matrix}\right]\right)\right]
\end{align}
\end{subequations}
where, $\mathbf{\nu}_{ij}^k$'s are normalized (using the homogeneity of~(\ref{eq:EpipolarRewrite1}) and the normalization function $\mathbf{\Theta}[\mathbf{x}] = \mathbf{x}/\|\mathbf{x}\|_2$, $\mathbf{\Theta}[\mathbf{0}] = \mathbf{0}$), and $f_i$ and $f_j$ denote the focal lengths of the $i$'th and $j$'th cameras, respectively. Hence, in the noiseless case, we make the following observation: Assuming $m_{ij}\geq2$ (and, that we can find at least two $\nu_{ij}^k$'s not parallel to each other), $\{\mathbf{\nu}_{ij}^k\}_{k=1}^{m_{ij}}$ determine a $2$D subspace to which $\mathbf{t}_i - \mathbf{t}_j$ is orthogonal (and hence the ``line'' through $\mathbf{t}_i$ and $\mathbf{t}_j$). This $2$D subspace spanned by $\{\mathbf{\nu}_{ij}^k\}_{k=1}^{m_{ij}}$ can be obtained by, e.g., principal component analysis (PCA) (i.e., as the subspace spanned by the two left singular vectors of the matrix $\left[\begin{smallmatrix}\mathbf{\nu}_{ij}^1 & \ldots & \mathbf{\nu}_{ij}^{m_{ij}}\end{smallmatrix}\right]$ corresponding to its two non-zero singular values). However, in the presence of noisy measurements, if we replace $R_i$'s, $f_i$'s and $\mathbf{q}_i$'s in~(\ref{eq:MeasEpiCst}) with their estimates, then we essentially obtain noisy samples $\hat{\mathbf{\nu}}_{ij}^k$'s from these subspaces (for which, e.g., PCA might not produce robust estimates in the presence of outliers among $\hat{\mathbf{\nu}}_{ij}^k$'s). Hence, for pairwise line estimation, our approach is to reduce the effects of noise by employing robust rotation and subspace estimation steps, that we discuss next.

\subsection{Rotation Estimation}
\label{sec:RotEstim}
In this section, we provide the details of the rotation estimation step using the pairwise rotation estimates $\hat{R}_{ij}$'s, extracted from the essential matrix estimates $\hat{E}_{ij}$'s. Our main objective here is to reduce the effects of outliers in $\hat{R}_{ij}$'s in the estimation of the rotations $R_i$, while preserving computational efficiency. The outliers in $\hat{R}_{ij}$'s, which mainly occur due to misidentified and/or small number of corresponding points used to estimate the essential matrices, can result in large errors in rotation estimation, which manifests itself as large errors in the pairwise line estimates through the noisy subspace samples $\hat{\nu}_{ij}^k$'s computed via~(\ref{eq:MeasEpiCst}). Specifically, for large and sparsely correlated image sets (i.e., image sets, for which we can obtain the estimates $\hat{R}_{ij}$ for a relative small fraction of all pairs), the proportion of outliers in $\hat{R}_{ij}$ is typically large enough to degrade the quality of rotation estimates (e.g., when estimated via a single iteration of the eigenvector method (EVM) used in~\cite{MicaAmitSfM}). In the literature, there are various algorithms to estimate camera rotations from the pairwise rotation estimates $\hat{R}_{ij}$'s (see, e.g.,~\cite{LanhuiSync, HartleyRotations, MicaAmitSfM, MartinecRotations}) with various theoretical and experimental robustness properties. Our procedure for estimating the rotations is iterative, where in each iteration we apply EVM in order to detect the outliers in $\hat{R}_{ij}$'s with respect to the rotation estimates evaluated at that iteration, and continue to the next iteration by removing the detected outliers. We now provide the details. 
\newline \indent We represent the pairwise rotation estimates $\hat{R}_{ij}$'s as the edges of a rotation measurement graph $G_R = \left(V_R,E_R\right)$, where the set of vertices $V_R$ represents the cameras. We assume that $G_R$ is connected. At the $k$'th iteration of our procedure, we are given a set $\{\hat{R}_{ij}\}_{(i,j)\in E_R^k}$ of pairwise rotation measurements represented by the connected rotation measurement graph $G_R^k = \left(V_R^k,E_R^k\right)$. First, we apply EVM to compute the $k$'th iteration rotation estimates $\{\hat{R}_i\}_{i\in V_R^k}$. We then identify the outlier edges, denoted by  $(i,j) \in E_O^k$, for which the consistency errors $\|(\hat{R}_i^k)^T\hat{R}_{j}^k - \hat{R}_{ij}^k\|_F$ are relatively large, i.e. which are ``significantly'' larger from the mean (one can also identify the outliers as a fixed portion, say $10\%$, of the edges with largest consistency error, or as the edges with a consistency error above a certain threshold, etc.) At this stage, we remove the outlier edges from the measurement graph and identify $G_R^{k+1}$ as the largest connected component of $G' = (V_R^k, E_R^k\setminus E_O^k)$. The iterations are then repeated until a convergent behavior in the consistency errors is observed (one can also repeat the iterations a fixed number of times, up to an allowed number of removed points, etc.) We note that since the eigenvector-based method is computationally very efficient, the extra iterations induce a negligible computational cost for the whole algorithm, however the change in the estimated rotations can significantly improve the final estimation of camera locations.

\subsection{Subspace Estimation}
\label{sec:SubspaceEstim}
Let $\bar{G}_R = (\bar{V}_R,\bar{E}_R)$ denote the resulting graph of the rotation estimation step. For each $(i,j)\in \bar{E}_R$, the estimates $\{\hat{\mathbf{\nu}}_{ij}^k\}_{k=1}^{m_{ij}}$ are evaluated using the rotation estimates $\{\hat{R}_i\}_{i\in \bar{V}_R}$ in~(\ref{eq:MeasEpiCst}). $\{\hat{\mathbf{\nu}}_{ij}^k\}_{k=1}^{m_{ij}}$ are noisy samples of unit vectors in a $2$D subspace. As also mentioned previously, we can estimate this subspace, e.g., by PCA, however PCA is not robust to outliers in the sample set. There are various algorithms for robust subspace estimation from noisy sample points (e.g. see~\cite{SSReaper,SSTylerM,SSGiannakis} and references therein), with different performance and convergence guarantees, computational efficiency, and numerical stability. We choose to use the S-REAPER algorithm introduced in~\cite{SSReaper}. S-REAPER solves the following convex problem:
\begin{subequations}
\label{eq:REAPER}
\begin{align}
\underset{{\scriptstyle Q_{ij}}}{\text{minimize}}
& \ \ \sum_{k=1}^{m_{ij}}  \| \hat{\mathbf{\nu}}_{ij}^k - Q_{ij}\hat{\mathbf{\nu}}_{ij}^k \|_2 \\
\text{subject to} & \ \ 0 \preceq Q_{ij} \preceq I_3, \ \tr\left(Q_{ij}\right) = 2 \,
\end{align}
\end{subequations}
After finding the solution $Q_{ij}^*$ of (\ref{eq:REAPER}), the robust subspace $\hat{Q}_{ij}$ is defined to be the subspace spanned by the two normalized leading eigenvectors, $\mathbf{q}_{1,ij}^*, \mathbf{q}_{2,ij}^*$, of $Q_{ij}^*$. Hence, we set $\hat{\Gamma}_{ij} \defeq I_3 - (\mathbf{q}_{1,ij}^*(\mathbf{q}_{1,ij}^*)^T + \mathbf{q}_{2,ij}^*(\mathbf{q}_{2,ij}^*)^T)$ as our robust line estimates.\\ \\
\indent A summary of our algorithm for camera motion estimation is provided in Table~\ref{tab:AlgorithmSummary}.
\begin{table}[!htbp]
\caption{Algorithm for camera motion estimation\label{tab:AlgorithmSummary}}
\begin{center}
\footnotesize
\tabcolsep=0.1cm
\begin{tabular}{L{2.5cm}||L{10cm}}
\hline 
& Input: Images: $\{\mbox{I}_i\}_{i=1}^n$, \ Focal lengths: $\{f_i\}_{i=1}^n$ \vspace{0.035in}\\ \hline\hline
Features Points, Essential Matrices, Relative Rotations  & {\bf 1.} Find corresponding points between image pairs (using SIFT~\cite{MicaAmit22}) \newline {\bf 2.} Compute essential matrices $\hat{E}_{ij}$, using RANSAC (for pairs with \newline \hspace{0.15in}sufficiently many correspondences) \newline {\bf 3.} Factorize $\hat{E}_{ij}$ to compute $\{\hat{R}_{ij}\}_{(i,j)\in E_{R}}$ and $G_R = (V_R,E_R)$ \vspace{0.035in} \\ \hline
Rotation Estimation (\S\ref{sec:RotEstim}) & {\bf 4.} Starting with $G^1_R = G_R$, at the $k$'th iteration: \newline \hspace{0.05in} {\bf -} For $G^k_R = (V_R^k,E_R^k)$ and $\{\hat{R}_{ij}\}_{(i,j)\in E^k_R}$, compute $\{\hat{R}_i\}_{i\in V_R^k}$ by EVM~\cite{MicaAmitSfM} \newline \hspace{0.05in} {\bf -} Detect outlier edges $E_O^k$ \newline \hspace{0.05in} {\bf -} Set $G^{k+1}_R$ to be the largest connected component of $G' = (V_R^k,E_R^ k\setminus E_O^k)$ \newline {\bf 5.} Repeat until convergence, output $\{\hat{R}_i\}_{i\in \bar{V}_R}$ and $\bar{G}_R = (\bar{V}_R,\bar{E}_R)$ \vspace{0.035in}\\ \hline
Pairwise Line Estimation (\S\ref{sec:SubspaceEstim})& {\bf 6.} Compute the $2$D subspace samples $\{\hat{\nu}_{ij}^k\}_{k=1}^{m_{ij}}$ for each $(i,j)\in \bar{E}_R$ (\ref{eq:MeasEpiCst}) \newline {\bf 7.} Estimate the pairwise lines $\{\hat{\Gamma}_{ij}\}_{(i,j)\in \bar{E}_{R}}$ using S-REAPER (\ref{eq:REAPER}) \vspace{0.035in}\\ \hline
Location Estimation (\S\ref{sec:TransEstim})& {\bf 8.} Extract the largest maximally parallel rigid component $G_t$ of $\bar{G}_R$~\cite{MaximalRigid} \newline {\bf 9.} If $|V_t|$ is small enough, estimate $\{\mathbf{t}_i\}_{i\in V_t}$ by the SDR (\ref{eq:SDRTransEst}) \newline {\bf 9}'{\bf.} If $|V_t|$ is large (w.r.t. the computational resources): \newline \hspace{0.05in} {\bf -} Partition $G_t$ into parallel rigid patches (\S\ref{sec:GraphPartition}), form the patch graph $G_P$ \newline \hspace{0.05in} {\bf -} Compute camera location estimates for each patch, using the SDR (\ref{eq:SDRTransEst}) \newline \hspace{0.05in} {\bf -} Compute pairwise patch signs, using (\ref{eq:PatchReg}), synchronize patches in $\mathbb{Z}_2$~\cite{AmitMihailSensors} \newline \hspace{0.05in} {\bf -} Estimate patch scales, translations and locations $\{\hat{\mathbf{t}}_{i}\}_{i\in \bigcup P_k}$, by (\ref{eq:JointlyScTrSynch}) \vspace{0.035in}\\ \hline\hline
 & Output: Camera orientations and translations: $\{\hat{R}_i, \hat{\mathbf{t}}_i\}$ \vspace{0.035in}\\
\hline
\end{tabular}
\end{center}
\end{table}

\section{Experiments} 
\label{sec:Experims}
\subsection{Synthetic Data Experiments}
\label{sec:SyntExp}
We first provide experimental evaluation of the performance of our SDR~(\ref{eq:SDRTransEst}) in $\R^3$ with synthetic line data. The experiments present the SDR performance with respect to the underlying graph (e.g., number of nodes, number of edges), and noise model and noise level. Also, we provide comparisons to the least squares (LS) method of~\cite{MicaAmitSfM,BATL2} and $\ell_{\infty}$ method of~\cite{HartleySim}, which also directly use pairwise line information in location estimation. Moreover, we compare the performance of our distributed algorithm to that of the SDR applied directly to the entire graph.\\
\indent We use a noise model incorporating the effects of small disturbances and outliers. Given a set of locations $\{\mathbf{t}_i\}_{i=1}^n$ and $G_t = (V_t,E_t)$, for each $(i,j)\in E_t$, we first let
\begin{equation}
\gamma_{ij} = \begin{cases} \gamma_{ij}^U \ \ , & \mbox{w.p. } \ p \\ (\mathbf{t}_i-\mathbf{t}_j)/\|\mathbf{t}_i-\mathbf{t}_j\|_2 + \sigma\gamma_{ij}^G \, & \mbox{w.p.} \ 1-p \end{cases}
\label{eq:NoiseModel}
\end{equation}
and normalize $\gamma_{ij}$'s to obtain $\{\tilde{\gamma}_{ij} = \mathbf{\Theta}[\gamma_{ij}]\}_{(i,j)\in E_t}$ as our ``directed'' lines. Here, $\{\gamma_{ij}^U\}_{(i,j)\in E_t}$  and $\{\gamma_{ij}^G\}_{(i,j)\in E_t}$ are i.i.d. random variables drawn from uniform distribution on $S^2$ and standard normal distribution on $\mathbb{R}^3$, respectively. For the SDR and the LS solvers, we use the (undirected) lines $\Gamma_{ij} = \tilde{\gamma}_{ij}(\tilde{\gamma}_{ij})^T$, while the $\ell_{\infty}$ solver requires the directed lines $\tilde{\gamma}_{ij}$.\\
\indent We evaluate the performance of each method in terms of the ``normalized root mean squared error'' (NRMSE) given by 
\begin{equation}
\label{eq:NRMSE}
\mbox{NRMSE}(\{\hat{\mathbf{t}}_i\}) = \sqrt{\frac{\sum_i \|\hat{\mathbf{t}}_i - \mathbf{t}_i\|_2^2} {\sum_i \|\mathbf{t}_i - \mathbf{t}_c\|_2^2}}
\end{equation}
where $\hat{\mathbf{t}}_i$'s are the location estimates (after removal of the global scale, translation and negation) and $\mathbf{t}_c$ is the center of mass of $\mathbf{t}_i$'s. \\
\indent We performed experiments with fixed parallel rigid graphs with $n = 100$ and $n = 200$ nodes having average and minimum degrees of $n/4$ and $3n/100$, respectively. The original locations $\mathbf{t}_i$'s are i.i.d. random variables drawn from standard normal distribution on $\mathbb{R}^3$. The NRMSE values (averaged over $10$ random realizations of the noise model and the original locations) are summarized in Table~\ref{tab:SynDataExp}. In order to observe the performance of the solvers with different noise structures (and since the LS and $\ell_{\infty}$ solvers are already very sensitive to outliers), we perform these experiments with pure small disturbances, i.e. with $p = 0$ (the first three rows of Table~\ref{tab:SynDataExp}) and with pure outliers, i.e. $\sigma  = 0$ (the last three rows of Table~\ref{tab:SynDataExp}).
\begin{table}[!htbp]
\caption{NRMSE~$(\ref{eq:NRMSE})$ performance of the SDR~$(\ref{eq:SDRTransEst})$ vs. least squares (LS)~{\rm \cite{MicaAmitSfM,BATL2}} and $\ell_{\infty}$~{\rm \cite{HartleySim}} solvers. Measurements are generated by the noise model~$(\ref{eq:NoiseModel})$ and NRMSE values are averaged over $10$ trials.\label{tab:SynDataExp}}
\begin{center}\footnotesize
\tabcolsep=0.3cm
\begin{tabular}{c|c|c|c||c|c|c|}
\cline{2-7}
& \multicolumn{3}{ ||c|| }{$n = 100$} & \multicolumn{3}{ c| }{$n = 200$} \\ \hline
\multicolumn{1}{ |c| }{$\sigma$} & \multicolumn{1}{ ||c| }{SDR} & LS & $\ell_{\infty}$ & SDR & LS & $\ell_{\infty}$ \\ \hline
\multicolumn{1}{|c| }{0.01}& \multicolumn{1}{ ||c| }{{\bf 0.0209}} & 0.0417 & 0.0619 & {\bf 0.0178} & 0.0525 &   0.0194  \\ \cline{1-7}
\multicolumn{1}{ |c| }{0.05}& \multicolumn{1}{ ||c| }{{\bf 0.0752}} & 1.1947 & 0.3742 & {\bf 0.0368} & 1.0760 &   0.7448  \\ \cline{1-7}
\multicolumn{1}{ |c| }{0.1}& \multicolumn{1}{ ||c| }{{\bf 0.1936}} & 1.2704 & 1.0247 & {\bf 0.1453} & 1.3870 &   0.8976  \\ \hline\hline
\multicolumn{1}{ |c| }{$p$} & \multicolumn{1}{ ||c| }{SDR} & LS & $\ell_{\infty}$ & SDR & LS & $\ell_{\infty}$ \\ \hline
\multicolumn{1}{ |c| }{0.01}& \multicolumn{1}{ ||c| }{{\bf 0.1049}} & 1.1584 & 1.1350 & {\bf 0.1189} & 1.1063 &   0.8326  \\ \cline{1-7}
\multicolumn{1}{|c|}{0.02}& \multicolumn{1}{ ||c| }{{\bf 0.1481}} & 1.1994 & 1.0876 & {\bf 0.1333} & 1.1226 &   1.0825  \\ \cline{1-7}
\multicolumn{1}{ |c| }{0.05}& \multicolumn{1}{ ||c| }{{\bf 0.2458}} & 1.2248 & 1.0689  & {\bf 0.2064} & 1.3848 &  1.1163    \\ \cline{1-7}
\end{tabular}
\end{center}
\end{table}\\
\indent Table~\ref{tab:SynDataExp} indicates that the estimator given by our SDR is robust to both types of noise, and increasing the number of nodes and edge connectivity further improves its accuracy. However, the LS and the $\ell_{\infty}$ solvers are sensitive to both kinds of noise. This is mainly due to the collapse phenomenon for the LS solver, and due to the structure of the cost function, which is not robust to large errors, for the $\ell_{\infty}$ solver. A collapsing solution of the LS solver is compared to the SDR solution in Figure~\ref{fig:CollapseExample} (the $\ell_{\infty}$ solution is not included since it produces a ``cloud'' of locations unrelated to the ground truth and degrades the visibility of the figure). 
\begin{figure}[!htbp]
\begin{center}
   \includegraphics[trim=2cm 1cm 2cm 1cm, clip=true, width=0.65\linewidth]{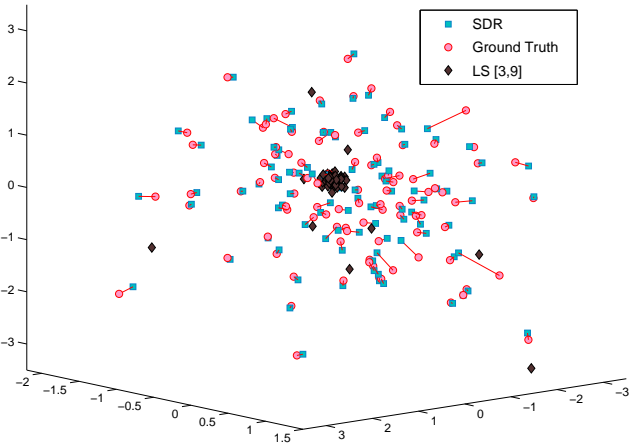}
\end{center}
   \caption{A sample solution, with $n = 100$, $p = 0$, $\sigma = 0.05$, demonstrating the collapsing behavior of the least squares (LS) solution. The line segments represent the error incurred by the SDR solution compared to the ground truth.\label{fig:CollapseExample}}
\end{figure}

\begin{table}[!htbp]
\caption{NRMSE~$(\ref{eq:NRMSE})$ results of the SDR solver~$(\ref{eq:SDRTransEst})$ (denoted by `Full') and the distributed SDR solver (denoted by `Dist.') (see $\S\ref{sec:DistApproach}$ for the distributed SDR).\label{tab:FullvsPartSDR}}
\begin{center}\footnotesize
\tabcolsep=0.3cm
\begin{tabular}{c|*{6}{c|}}\cline{2-7}
& \multicolumn{2}{ |c| }{$p = 0.01$}& \multicolumn{2}{ |c| }{$p = 0.02$} & \multicolumn{2}{ |c| }{$p = 0.05$} \\\hline
\multicolumn{1}{ |c| }{$\sigma$}&Full & Dist. & Full & Dist. & Full & Dist.\\ \hline 
\multicolumn{1}{ |c| }{0.01} &{\bf 0.1089} &0.1175 &{\bf 0.1255} & 0.1342 &0.1957 &{\bf 0.1852}\\\hline
\multicolumn{1}{ |c| }{0.02} &{\bf 0.1174} & 0.1207 &0.1381 &{\bf 0.1364} &0.2064 &{\bf 0.1960}\\\hline
\multicolumn{1}{ |c| }{0.05} &0.1426 &{\bf 0.1385} &{\bf 0.1490} & 0.1523 &0.2137 &{\bf 0.2058}\\\hline
\end{tabular}
\end{center}
\end{table}
\indent We also compare the performance of our distributed algorithm to that of the SDR applied to the whole graph for $n = 200$. For the distributed algorithm, we fix a maximum patch size of $N_{max} = 70$ and divide the whole graph into $8$ patches. 
The NRMSE results, summarized in Table~\ref{tab:FullvsPartSDR}, demonstrate that the accuracy of the distributed approach is comparable to that of the non-distributed approach, perhaps even slightly better for higher levels of noise.

\subsection{Real Data Experiments}
\label{sec:RealDataExp}
We tested our location estimation algorithm on three sets of real images with different sizes. To solve the SDR~(\ref{eq:SDRTransEst}), we use the SDPT3 package from~\cite{SDPT32} for small data sets (for its super-linear convergence with respect to the level of accuracy), and for the large data set, we use ADM (Algorithm~\ref{alg:ADM}). To construct a sparse $3$D structure in our experiments, we use the parallel bundle adjustment (PBA) algorithm of~\cite{PBA}. We also use the patch-based multi-view stereo (PMVS) algorithm of~\cite{Dense3D} to evaluate a dense $3$D reconstruction. We perform our computations on multiple workstations with Intel(R) Xeon(R) X$7542$ CPUs, each with $6$ cores, running at $2.67$ GHz. In order to directly compare the accuracy of the location estimation by SDR to that of least squares (LS)~\cite{MicaAmitSfM,BATL2} and $\ell_{\infty}$~\cite{HartleySim} solvers, we feed all three solvers with the orientation estimates produced by our iterative solver (\S\ref{sec:RotEstim}) and the robust subspace estimates (\S\ref{sec:SubspaceEstim}), that produced more accurate estimates for all data sets.

\subsection{Small Data Sets}
We first provide our results for the small Fountain-P$11$ and HerzJesu-{P}$25$ data sets of~\cite{FountainData}, which include $11$ and $25$ images, respectively. For camera calibration and performance evaluation, we use the focal lengths and ground truth locations provided with the data. For these small data sets, we do not require a distributed approach. For both data sets, the SDR~(\ref{eq:SDRTransEst}) estimates the camera locations very accurately in less than a second and the solutions of~(\ref{eq:SDRTransEst}) have rank equal to $1$, demonstrating the tightness of our relaxation. We also run the PBA algorithm, resulting in an average reprojection error of $0.17$ pixels for the Fountain-P$11$ data and $0.43$ pixels for the HerzJesu-{P}$25$ data, to construct the $3$D structures given in Figure~\ref{fig:Fountain3D}. We summarize and compare our end results to previous works\footnote{Results of~\cite{MultiLinear} are cited from~\cite{MicaAmitSfM}} in Table~\ref{tab:FountainRes}.
\begin{table}[!htbp]
\caption{Location estimation errors for the Fountain-{P$11$} and the HerzJesu-{P}$25$ data sets.\label{tab:FountainRes}}
\begin{center}\footnotesize
\begin{tabular}{|l|c|c|}
\hline
& \multicolumn{2}{ c| }{Error (in meters)} \\
\hline
Method & Fountain-P$11$ & HerzJesu-{P}$25$ \\
\hline
SDR~(\ref{eq:SDRTransEst})  & {\bf 0.0002} & {\bf 0.0053} \\
LS~\cite{MicaAmitSfM,BATL2} & 0.0048 & 0.0054\\
$\ell_{\infty}$~\cite{HartleySim} & 0.0064 & 0.0253 \\
Linear method of~\cite{MultiLinear} & 0.1317 & 0.2538\\
Bundler~\cite{SnavelySkeletal}& 0.0072 & 0.0308\\
\hline
\end{tabular}
\end{center}
\end{table}

\begin{figure}[!htbp]
\centering
\subfloat[]{\vspace{0.1in}
\includegraphics[trim=12cm 3cm 14cm 3cm, clip=true, width=0.4\linewidth]{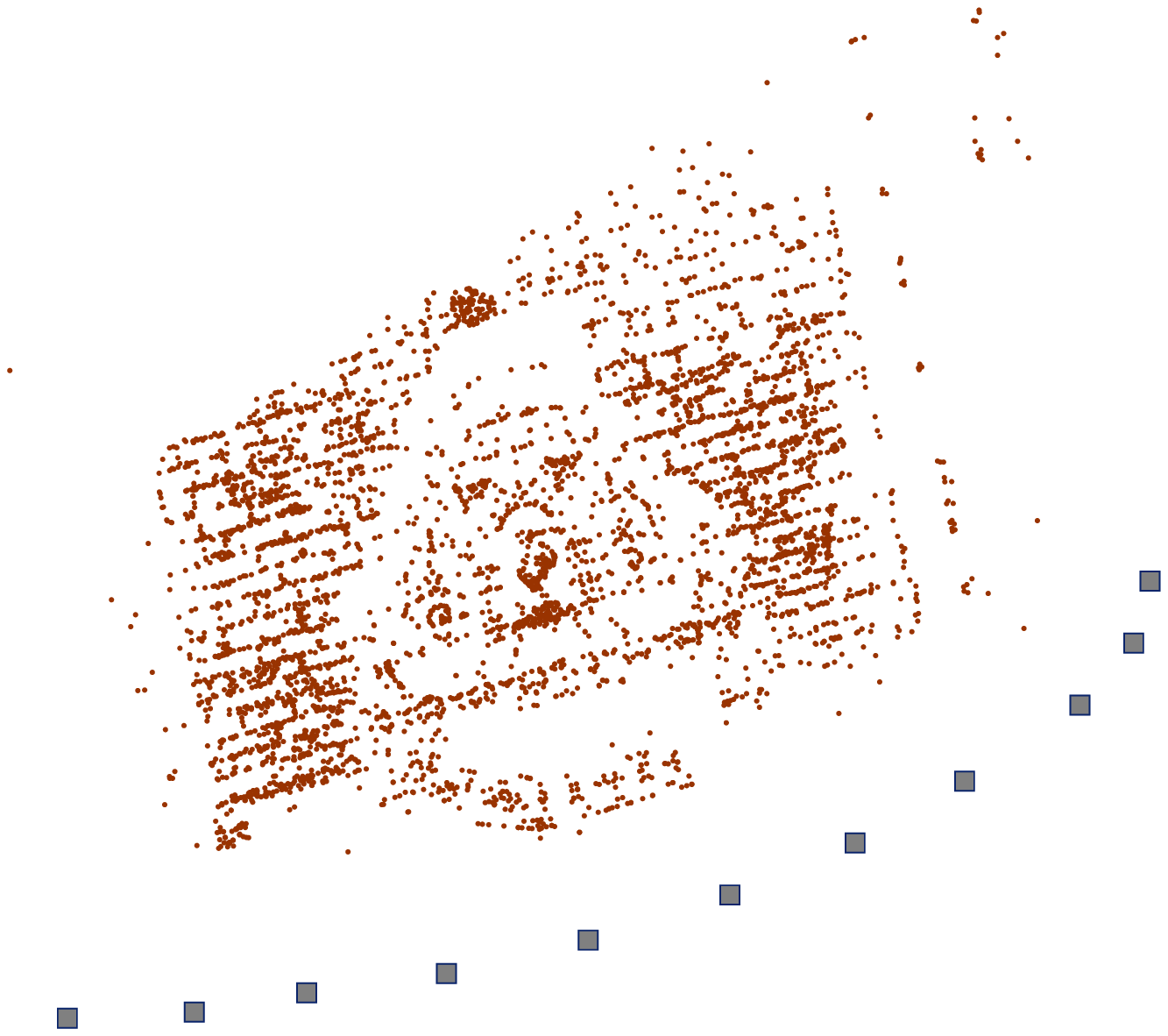}} \hspace{0.18in}
\subfloat[]{\includegraphics[trim=0cm 0cm 0cm 0cm, clip=true, width=0.45\linewidth]{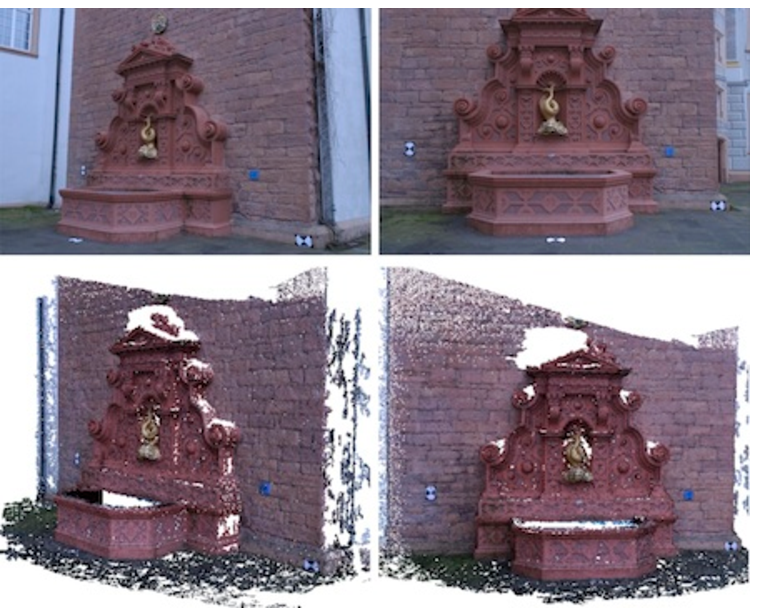}} \\ \vspace{0.1in}
\subfloat[]{\includegraphics[trim=5cm 2cm 7cm 2cm, clip=true, width=0.43\linewidth]{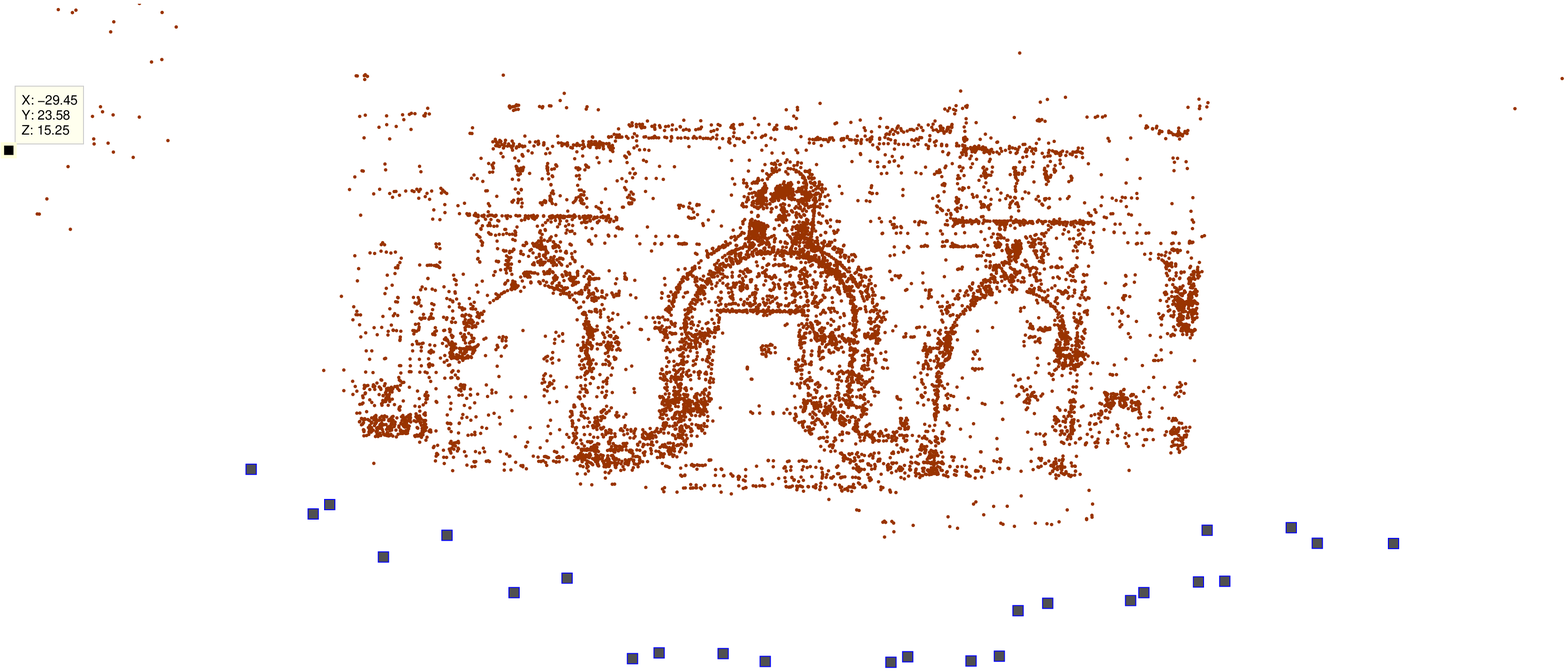}}\hspace{0.1in}
\subfloat[]{\includegraphics[trim=0cm 0cm 0cm 0cm, clip=true, width=0.45\linewidth]{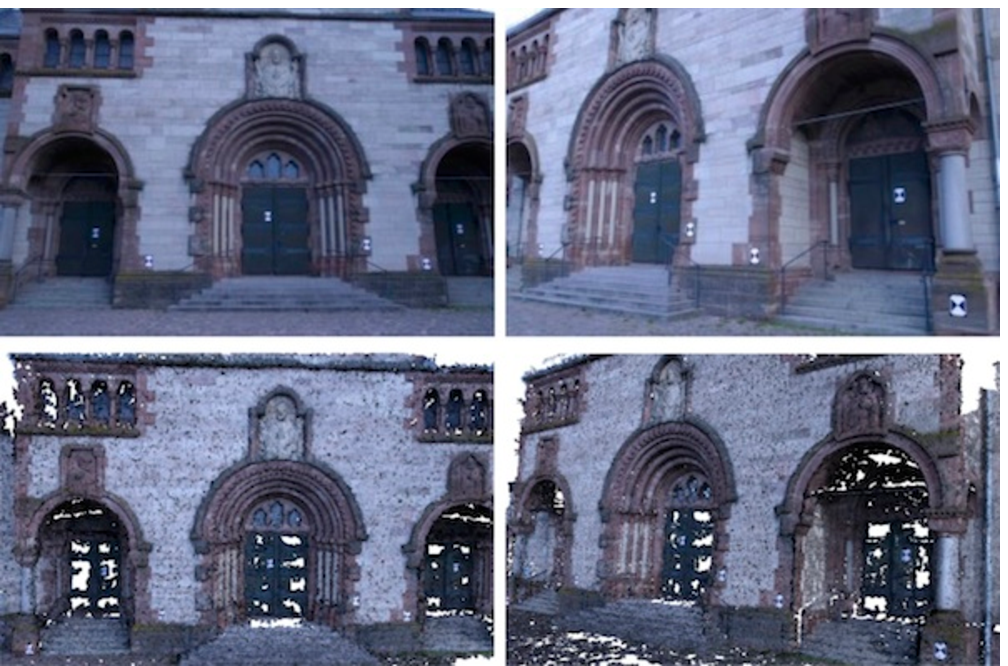}}
 \caption{{\rm(a)} Sparse $3$D structure and $11$ estimated camera locations (in blue) for the Fountain-{P}$11$ data {\rm(b)} Sample images and snapshots of the dense $3$D reconstruction {\rm(c)} Sparse $3$D structure and $25$ estimated camera locations (in blue) for the HerzJesu-{P}$25$ data {\rm(d)} Sample images and snapshots of the dense $3$D reconstruction\label{fig:Fountain3D}}
\end{figure}

\subsection{Large Data Set}
The next set of images is the larger Notre-Dame Cathedral set from~\cite{SnavelyData}, composed of $715$ images. This is an irregular collection of images and hence estimating the camera locations for all of these images (or a large subset) is challenging. Even for orientation estimation on this data set, previous works usually discard a large subset (see, e.g.~\cite{HartleyRotations, MicaAmitSfM}). In our experiment, we use the focal lengths and radial distortion corrections from~\cite{SnavelyData}. We can accurately and efficiently estimate the camera locations for a subset of size $637$. This reduction in size is due to our rotation estimation step (\S\ref{sec:RotEstim}) (we can robustly recover rotations for a subset of size $687$) and due to the node removals during the distributed algorithm (\S\ref{sec:DistApproach}). We partition the whole graph into patches of sizes smaller than $150$ (\S\ref{sec:GraphPartition}), getting $20$ rigid patches in less than a minute (where, extraction of parallel rigid components takes a total of $42$ secs). The related SDRs are solved in parallel, in about $21$ mins. Finally, the stitching of the patches (\S\ref{sec:PatchReg} and \S\ref{sec:GlobalStitch}) takes $57$ secs. \\
\indent We assess the performance of our algorithm in terms of the NRMSE measure~(\ref{eq:NRMSE}), using the location estimates of~\cite{SnavelyData} as the ground truth. We first get an NRMSE of $0.104$, and then apply PBA once (with an initial $3$D structure of about $204$K points) to get an NRMSE of $0.054$, and an average reprojection error of $0.43$ pixels, in less than $2$ mins. The resulting $3$D structure is also provided in Figure~\ref{fig:Notre3D}. \\
\begin{figure}[!htbp]
\centering
\subfloat[]{\includegraphics[trim=0cm 0cm 0cm 0cm, clip=true, width=0.49\linewidth]{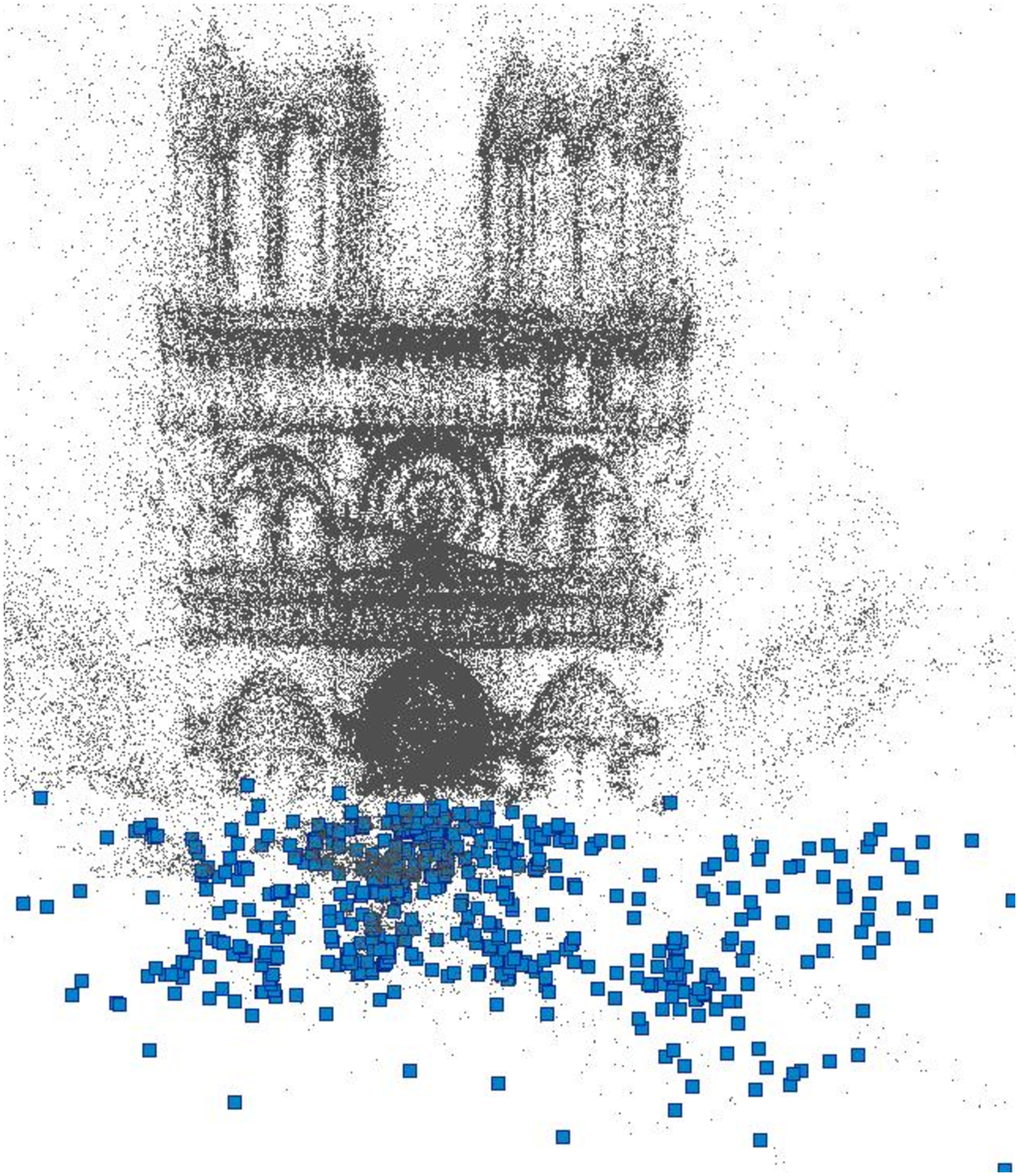}}\hspace{0.15in}
\subfloat[]{\includegraphics[trim=0cm 0cm 0cm 0cm, clip=true, width=0.42\linewidth]{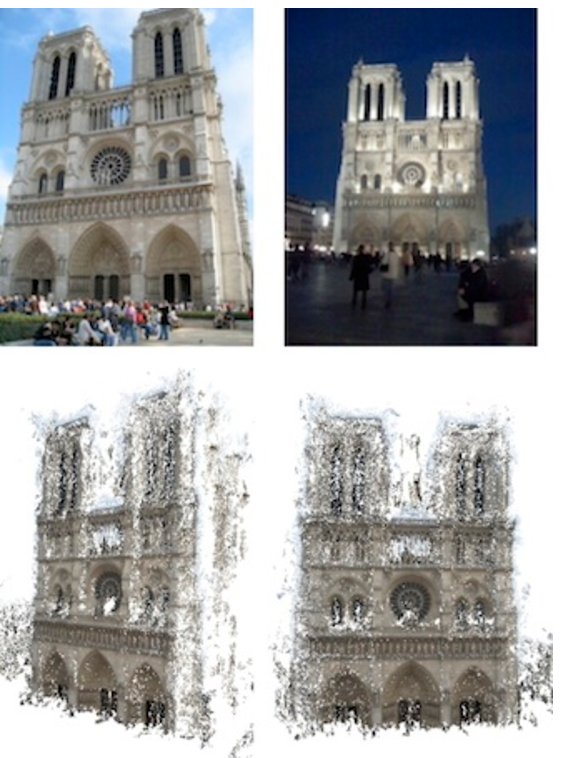}}
 \caption{{\rm(a)} Sparse $3$D structure of $\mathord{\sim}204$K points and (some of) $637$ estimated camera locations (in blue) for the Notre-Dame data {\rm(b)} Sample images and snapshots of the dense $3$D reconstruction\label{fig:Notre3D}}
\end{figure}
\indent We also compare our results to those of the LS and $\ell_{\infty}$ solvers applied to the whole graph and in a distributed fashion. In both cases, the LS and $\ell_{\infty}$ solvers resulted in very large errors\footnote{We note that (a slightly modified version of) the least squares solver in~\cite{MicaAmitSfM} achieves reasonable accuracy for the Notre-Dame data set when a significant number of the images are discarded.} (also, because of the very low quality of initial $3$D structures for these solutions, PBA produced no improvements). The NRMSE results are summarized in Table~\ref{tab:NotreRes}. Also, a snapshot of the location estimates for our distributed algorithm and the LS solver, demonstrating the collapsing LS solution, are provided in Figure~\ref{fig:NotreSDRvsL2}.
\begin{figure}[!htbp]
\begin{center}
\includegraphics[trim=1.5cm 0.75cm 1.4cm 0cm, clip=true, width=0.7\linewidth]{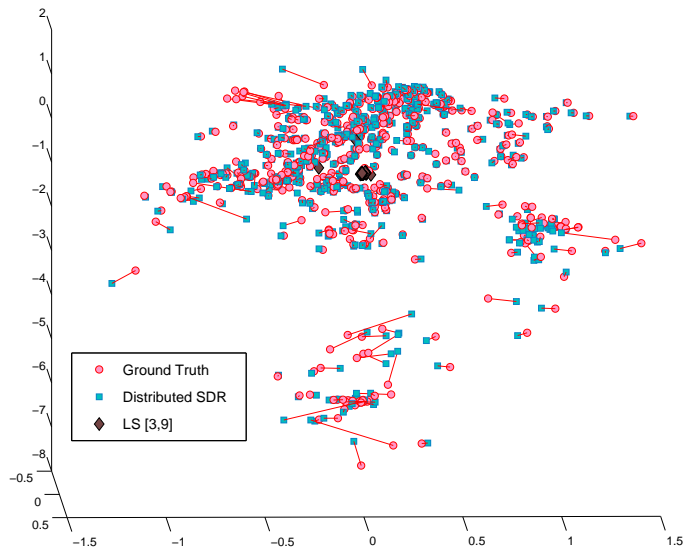}
\end{center}
   \caption{$637$ camera locations of the Notre-Dame data set (with $715$ images) estimated using the distributed SDR. The solution of~{\rm\cite{SnavelyData}} is taken as the ground truth. The collapsing solution of the least squares (LS) solver of~{\rm\cite{MicaAmitSfM,BATL2}} is also provided.\label{fig:NotreSDRvsL2}}
\end{figure}

\begin{table}[!htbp]
\caption{Location estimation errors for the Notre-Dame data.\label{tab:NotreRes}}
\begin{center}\footnotesize
\begin{tabular}{|l|c|}
\hline
Method & NRMSE \\
\hline\hline
Distributed SDR & 0.104 \\
Distributed SDR (followed by PBA) & 0.054 \\
Distributed LS~\cite{MicaAmitSfM,BATL2} & 1.087\\
LS~\cite{MicaAmitSfM,BATL2} & 1.392\\
Distributed $\ell_{\infty}$~\cite{HartleySim}  & 1.125\\
$\ell_{\infty}$~\cite{HartleySim}  & 1.273\\
\hline
\end{tabular}
\end{center}
\end{table}

\section{Conclusion and Future Work}
\label{sec:Conclusion}
We formulated a semidefinite relaxation for estimating positions from pairwise line measurements, and applied it to the problem of camera location estimation in structure from motion. We elucidated the importance of parallel rigidity in determining uniqueness and well-posedness of the estimation problem and provided rigorous analysis for stability of our SDR. Also, we introduced an alternating direction method to solve the SDR and an efficient, distributed version of our SDR designed to handle very large sets of input images. 
\newline \indent In the context of structure from motion, our algorithm can be used to efficiently and stably obtain the locations and orientations of the corresponding cameras, i.e. the camera motion, in order to produce a high-quality initial point for reprojection error minimization algorithms, as demonstrated by our experiments on real image sets. We also note that, for collections of images taken from an approximately planar surface (which is usually the case for images of large $3$D structures taken by visitors), the pairwise lines can also be well-approximated to lie within the same plane, and hence the SDR can take this prior knowledge to allow a more efficient (and perhaps, more accurate) computation of camera locations rendered by the reduction of the problem from $3$D to $2$D. Also in SfM, since the sign information (i.e. the directions) of the pairwise lines between camera locations can be estimated, an optimization framework that uses this additional information is a topic for further study (see, e.g., \cite{LUDSfM,TronVidalDist}).
\newline \indent As future work, we plan to investigate and explain the tightness of the SDR, i.e. to characterize the conditions under which it returns a rank-$1$ matrix as its solution. Also, we plan to apply the SDR to the problem of sensor network localization with bearing information.

\appendix
\section{Parallel Rigidity}
\label{Apdx:ParallelRigidity}
In this appendix, we review fundamental concepts and results in parallel rigidity theory (also see~\cite{ErenNetwork,ErenNetwork2,WhiteleyMatroidBook,WhiteleyMatroid} and the references therein). We begin with the concept of {\em point formation}. A $d$-dimensional point formation $\mathbb{F}_{\mathbf{p}}$ at $\mathbf{p} = \left[\begin{smallmatrix} \mathbf{p}_1^T & \mathbf{p}_2^T & \ldots & \mathbf{p}_n^T\end{smallmatrix}\right]^T$ is a set of $n$ points $\left\{\mathbf{p}_1,\mathbf{p}_2,\ldots,\mathbf{p}_n\right\} \subseteq \R^d$ (assumed to be separate, i.e. $\mathbf{p}_i \neq \mathbf{p}_j$), together with a set $E$ of {\em links}, satisfying $E \subseteq \left\{(i,j) : i \neq j, \ i,j \in \{1,2,\ldots,n\} \right\}$. For the camera location estimation problem, we think of the points $\mathbf{p}_i$ as representing the camera locations and the pairs $(i,j)\in E$ are used to denote the camera pairs, between which there is a pairwise measurement. Note that each formation $\mathbb{F}_{\mathbf{p}}$ uniquely determines a {\em graph} $G_{\mathbb{F}_{\mathbf{p}}} \defeq (V,E)$, having the vertex set $V = \left\{1,2,\ldots,n\right\}$ and the edge set $E$ of $\mathbb{F}_{\mathbf{p}}$, and also a {\em measurement function} $\Gamma_{\mathbb{F}_{\mathbf{p}}}$, whose value at $(i,j)\in E$ is the measured quantity between $\mathbf{p}_i$ and $\mathbf{p}_j$ (to keep the notation simple in \S\ref{sec:ParRigidity}, by abuse of notation, we refer to the set $\{\Gamma_{ij}\}_{(i,j)\in E_t}$ of measurements $\Gamma_{ij}\defeq \Gamma_{\mathbb{F}_{\mathbf{p}}}(i,j)$, defined on $E_t$ of $G_t = (V_t,E_t)$, as {\em the formation}). In order to represent the {\em pairwise line measurements}, we use a measurement function given by $\Gamma_{\mathbb{F}_{\mathbf{p}}}(i,j) = (\mathbf{p}_i - \mathbf{p}_j)(\mathbf{p}_i - \mathbf{p}_j)^T/\|\mathbf{p}_i - \mathbf{p}_j\|_2^2$. In the literature (see, e.g.,~\cite{ErenNetwork,ErenNetwork2}), the pairwise measurements are considered only in terms of the {\em direction constraints} they imply. These constraints are used to define {\em parallel point formations} $\mathbb{F}_{\mathbf{q}}$ of a formation $\mathbb{F}_{\mathbf{p}}$ (explained below), and are homogeneous equations given by
\begin{eqnarray}
\label{eq:ClassicalDirContrnts}
(\mathbf{p}_i - \mathbf{p}_j)_{N_1}^T(\mathbf{q}_i-\mathbf{q}_j) &=& 0 , \   (i,j) \in E \nonumber \\
&\vdots& \nonumber \\
(\mathbf{p}_i - \mathbf{p}_j)_{N_{d-1}}^T(\mathbf{q}_i-\mathbf{q}_j) &=& 0 , \   (i,j) \in E \nonumber
\end{eqnarray}
where $(\mathbf{p}_i - \mathbf{p}_j)_{N_i}$, for $i =1,\ldots, d-1$, are (linearly independent) vectors that span the subspace orthogonal to $p_i-p_j$, and therefore, uniquely define the line between $\mathbf{p}_i$ and $\mathbf{p}_j$. We use the measurement function $\Gamma_{\mathbb{F}_{\mathbf{p}}}$ to compactly represent these equations by $(I_d - \Gamma_{\mathbb{F}_{\mathbf{p}}}(i,j))(\mathbf{q}_i-\mathbf{q}_j) = 0, (i,j) \in E$. Also, note that, for the camera location estimation problem, $\Gamma_{\mathbb{F}_{\mathbf{p}}}$ encapsulates the maximal information we can extract from the epipolar constraints, e.g., we cannot estimate the {\em signed pairwise lines} $(\mathbf{p}_i - \mathbf{p}_j)/\|\mathbf{p}_i - \mathbf{p}_j\|_2$ based solely on the epipolar constraints (see \S\ref{sec:CamMotEst} for further details).\\
\indent Two point sets $\left\{\mathbf{p}_1,\mathbf{p}_2,\ldots,\mathbf{p}_n\right\}$ and $\left\{\mathbf{q}_1,\mathbf{q}_2,\ldots,\mathbf{q}_n\right\}$ in $\R^d$ are said to be {\em congruent} if there exist $\mathbf{x}\in\R^d$ and $c\in\R,$ such that $c\mathbf{p}_i + \mathbf{x} = \mathbf{q}_i$, for all $i\in\{1,2,\ldots,n\}$, i.e. congruent point sets can be obtained from one another by translation, scaling or negation. A point formation $\mathbb{F}_{\mathbf{p}}$ that is uniquely determined, up to congruence, by its graph $G_{\mathbb{F}_{\mathbf{p}}}$ and its measurement function $\Gamma_{\mathbb{F}_{\mathbf{p}}}$ is called {\em globally parallel rigid}.\\
\indent For a given formation $\mathbb{F}_{\mathbf{p}}$ in $\R^d$, a {\em parallel point formation} $\mathbb{F}_{\mathbf{q}}$ is a point formation (with the same graph $G_{\mathbb{F}_{\mathbf{p}}} = (V,E)$) in which $\mathbf{p}_i - \mathbf{p}_j$ is parallel to $\mathbf{q}_i - \mathbf{q}_j$, for all $(i,j)$ in $E$ (i.e., $\Gamma_{\mathbb{F}_{\mathbf{p}}}(i,j) = \Gamma_{\mathbb{F}_{\mathbf{q}}}(i,j)$ on $E$). It is clear that congruence transformations, i.e. translations, scalings and negation, produce {\em trivial} parallel point formations of the original point formation, any other parallel formation is termed {\em non-trivial}. A point formation that does not admit any non-trivial parallel point formations is called a {\em parallel rigid} point formation, otherwise it is called {\em flexible} (see Figure~\ref{fig:ParallelRigidityEx} for a simple example). We note that, in contrast to the case of classical rigidity involving distance measurements, equivalence of global parallel rigidity and (simple) parallel rigidity turns out to be a rephrasing of definitions (also see~\cite{WhiteleyMatroidBook,JacksonJordanSurvey,KatzUnique,ErenNetwork}). \\
\indent The concept of parallel point formations allows us to obtain a linear algebraic characterization: Given a point formation $\mathbb{F}_{\mathbf{p}}$ with the graph $G_{\mathbb{F}_{\mathbf{p}}} =  (V,E)$, $\mathbb{F}_{\mathbf{q}}$ is a parallel formation if and only if its point set satisfies 
\begin{equation}
\label{eq:ParallelFormations}
 \left(I_d - \Gamma_{\mathbb{F}_{\mathbf{p}}}(i,j)\right)\left(\mathbf{q}_i-\mathbf{q}_j\right) = 0 \ , \   (i,j) \in E
\end{equation}
which can be rewritten in matrix form as
\begin{equation}
\label{eq:ParallelFormationsMatrixForm}
 R_{\mathbb{F}_{\mathbf{p}}}\mathbf{q} = 0
\end{equation}
where, $R_{\mathbb{F}_{\mathbf{p}}} \in \R^{d|E|\times d|V|}$ is termed {\em the parallel rigidity matrix} of the formation $\mathbb{F}_{\mathbf{p}}$ (see, e.g.~\cite{ErenNetwork}, for a slightly different, but equivalent formulation). Here, point sets of the trivial parallel formations of $\mathbb{F}_{\mathbf{p}}$ span a $d+1$ dimensional subspace of the null space of $R_{\mathbb{F}_{\mathbf{p}}}$. As a result, $\mathbb{F}_{\mathbf{p}}$ is parallel rigid if and only if $\mbox{dim}\left(\mathcal{N}\left(R_{\mathbb{F}_{\mathbf{p}}}\right)\right) = d+1$, i.e. $\rank\left(R_{\mathbb{F}_{\mathbf{p}}}\right) = d|V| - (d+1)$ (note that, $R_{\mathbb{F}_{\mathbf{p}}}^TR_{\mathbb{F}_{\mathbf{p}}}$ is the matrix $L$ of the linear cost function in (\ref{eq:SDRTransEst}) with noiseless measurements). \\ 
\indent Now, we consider the generic properties of formations. A point $\mathbf{p}$ in $\R^{dn}$ (or the point set $\{\mathbf{p}_1,\mathbf{p}_2,\ldots,\mathbf{p}_n\}$ in $\R^d$) is said to be {\em generic} if its $dn$ coordinates are algebraically independent, i.e. there is no non-zero polynomial $\psi$ on $\R^{dn}$, with integer coefficients, satisfying $\psi(\mathbf{p}) = 0$ (for a more general definition see~\cite{GSSbook}). The set of generic points is an open dense subset of $\R^{dn}$. A graph $G = (V,E)$ is called {\em generically parallel rigid} (in $\R^d$) if, for a generic point $\mathbf{p}\in\R^{d|V|}$, the formation $\mathbb{F}_{\mathbf{p}}$ having the underlying graph $G$ is parallel rigid (in fact, if $G$ of $\mathbb{F}_{\mathbf{p}}$ is generically parallel rigid, then $\mathbb{F}_{\mathbf{p}}$ is parallel rigid for all generic $\mathbf{p}$). Also, a formation $\mathbb{F}_{\mathbf{p}}$ is called a {\em generically parallel rigid} formation (in $\R^d$) if its underlying graph is generically parallel rigid (in $\R^d$). Generic parallel rigidity is a combinatorial property of the underlying graph as characterized by Theorem~\ref{thm:LamanConds} (also see~\cite{ErenNetwork,ErenNetwork2,KatzUnique,WhiteleyMatroidBook,WhiteleyMatroid,JacksonJordanSurvey}). Combinatorial conditions of Theorem~\ref{thm:LamanConds} also translate to efficient combinatorial algorithms for generic parallel rigidity testing (see, e.g.,~\cite{PebbleGame}). \\
\indent There is also a linear algebraic characterization of generically parallel rigid formations using the notion of {\em generic rank}: The generic rank of the parallel rigidity matrix $R$ is given by $\mbox{GenericRank}(R) = \max\{\rank(R_{\mathbb{F}_{\mathbf{p}}}), \ \mathbf{p}\in\R^{d|V|}\}$, which clearly depends only on the underlying graph. A formation $\mathbb{F}_{\mathbf{p}}$ (and hence its underlying graph $G = (V,E)$) is generically parallel rigid in $\R^d$ if and only if $\mbox{GenericRank}(R) = d|V|-(d+1)$. Also, the set of points with maximal rank, i.e. $\{\mathbf{p}\in\R^{d|V|}: \rank(R_{\mathbb{F}_{\mathbf{p}}}) = \mbox{GenericRank}(R)\}$, is an open dense subset of $\R^{d|V|}$. As a result, similar to the randomized algorithm of~\cite{HendricksonRigid} for (classical) generic rigidity testing, we propose the following randomized test for generic parallel rigidity: 
\begin{remunerate}
\item Given a formation on the graph $G = (V,E)$, randomly sample $\mathbf{p}\in\R^{d|V|}$ from an absolutely continuous (w.r.t.~Lebesgue measure) probability measure (e.g., let $\mathbf{p}$ be sampled from i.i.d. Gaussian distribution) and centralize $\mathbf{p}$ such that $\sum_{i=1}^{|V|} \mathbf{p}_i = \mathbf{0}_d$.
\item Construct an orthogonal basis for the trivial null-space of $R_{\mathbb{F}_{\mathbf{p}}}$: For $\mathbf{u}_i = \mathbf{1}_{|V|}\otimes\mathbf{e}_i$ ($\mathbf{e}_i$ denoting the $i$'th canonical basis vector in $\R^d$), such a basis is given by $\{\mathbf{u}_1,\mathbf{u}_2,\ldots,\mathbf{u}_d,\mathbf{p}\}$.
\item To check if $\rank(R_{\mathbb{F}_{\mathbf{p}}}) = d|V|-(d+1)$ or not, compute the smallest eigenvalue $\lambda_{1}(W_{\mathbb{F}_{\mathbf{p}}})$ of $W_{\mathbb{F}_{\mathbf{p}}} = R_{\mathbb{F}_{\mathbf{p}}}^TR_{\mathbb{F}_{\mathbf{p}}} + U_{\mathbf{p}}U_{\mathbf{p}}^T$, where $U_{\mathbf{p}} = \left[\begin{smallmatrix}\mathbf{u}_1&\mathbf{u}_2&\ldots&\mathbf{u}_d&\mathbf{p}\end{smallmatrix}\right]$. If $\lambda_{1}(W_{\mathbb{F}_{\mathbf{p}}}) > \epsilon$ (where $\epsilon$ is a small positive constant set to prevent numerical precision errors), declare $G$ to be generically parallel rigid, otherwise declare $G$ to be flexible.
\end{remunerate}

\indent This randomized test correctly decides (up to precision errors) in the generic rigidity of $G$ with probability $1$ (i.e., up to precision errors, our procedure can produce a false negative with probability $0$, and does not produce a false positive). Also note that, $\lambda_{1}(W_{\mathbb{F}_{\mathbf{p}}}) = \lambda_{dn}(W_{\mathbb{F}_{\mathbf{p}}}) + \lambda_{1}(W_{\mathbb{F}_{\mathbf{p}}} - \lambda_{dn}(W_{\mathbb{F}_{\mathbf{p}}})I_{dn})$, and since $W_{\mathbb{F}_{\mathbf{p}}}$ is positive semidefinite, $\lambda_{dn}(W_{\mathbb{F}_{\mathbf{p}}})$ and $\lambda_{1}(W_{\mathbb{F}_{\mathbf{p}}} - \lambda_{dn}(W_{\mathbb{F}_{\mathbf{p}}})I_{dn})$ are largest magnitude eigenvalues, which can be computed (e.g., by the power method\footnote{Although every iteration of the power method has time complexity $\mathcal{O}(m)$, the number of iterations is greater than $\mathcal{O}(1)$ as it depends on the spectral gap of $W_{\mathbb{F}_{\mathbf{p}}}$.}) in $\mathcal{O}(m)$ time. For fixed $d$, this yields a time complexity of $\mathcal{O}(m)$ (dominated by the complexity of step $2$) for the randomized test (compare to the time complexity $O(n^2)$ of the pebble game algorithm, which is, however, an integer algorithm, i.e. is free of numerical errors). Also considering its simplicity, we choose to use this test for testing parallel rigidity.
\section{Proof of Theorem~\ref{thm:WeakStability}}
\label{Apdx:SDRStability}
Let $\tilde{c}$ denote the suboptimal constant for (\ref{eq:ErrorMeasure}) given by $\tilde{c} = (\tr(T^*))^{-1}$. We first consider the decomposition of $\tilde{c}T^*$ in terms of the null space $\mathcal{S}$ of $L^0$ and its complement $\bar{\mathcal{S}}$, that is we let 
\begin{equation}
\label{eq:Tdecompose}
\tilde{c}T^* = X + Y + Z
\end{equation}
where $X \in \mathcal{S}\otimes\mathcal{S}$, $Y \in \bar{\mathcal{S}}\otimes\bar{\mathcal{S}}$ and $Z \in (\bar{\mathcal{S}}\otimes\mathcal{S})\oplus(\mathcal{S}\otimes\bar{\mathcal{S}})$ \footnote{For two subspaces $\mathcal{U}$ and $\mathcal{V}$, the tensor product $\mathcal{U}\otimes \mathcal{V}$ denotes the space of matrices spanned by $\rank$-1 matrices $\{\mathbf{u}\mathbf{v}^T: \mathbf{u}\in\mathcal{U}, \ \mathbf{v}\in\mathcal{V}\}$.}. Using this decomposition, we can bound $\delta(T^*,T_0)$ in terms of $\tr(Y)$:
\begin{lemma}
\label{lem:TraceBnd}
\begin{equation}
\label{eq:TraceBnd}
\delta(T^*,T_0) \leq \|\tilde{c}T^* - T_0\|_F \leq \sqrt{2\tr(Y)} 
\end{equation}
\end{lemma}
\begin{proof}
First we note that the orthogonal decomposition (\ref{eq:Tdecompose}) and $T_0\in\mathcal{S}\otimes\mathcal{S}$ implies
\begin{equation}
\label{eq:SumSqrs}
\delta^2(T^*,T_0) \leq \|\tilde{c}T^* - T_0\|_F^2 = \|X - T_0\|_F^2 + \|Y\|_F^2 + \|Z\|_F^2
\end{equation}
Here, since the formation is parallel rigid, we get $dim(\mathcal{S}) = d+1$, where $\mathcal{S}$ is spanned by the $d$ eigenvectors of $H$ (corresponding to its only nonzero eigenvalue $n$) and $\mathbf{t}^0$. Since $\tr(\tilde{c}T^*H) = 0$, we get $X = a\mathbf{t}^0(\mathbf{t}^0)^T$ (with $1\geq a = \tilde{c}(\mathbf{t^0})^TT^*\mathbf{t}^0\geq0$). Also, since $\tr(\tilde{c}T^*) = 1$ and $\tr(Z) = 0$, we get
\begin{equation}
\label{eq:XnormBnd}
\tr(Y) = 1 - \tr(X) = 1-a = \tr((1-a)\mathbf{t}^0(\mathbf{t}^0)^T) = \tr(T_0 - X) = \|X-T_0\|_F
\end{equation}
We now consider the spectral decomposition of $L^0$ given by $L^0 = U_0\Sigma_0 U_0^T$, where $U_0 = [\begin{smallmatrix}\mathbf{v}_1&\mathbf{v}_2&\ldots&\mathbf{v}_d&\mathbf{t}^0&\mathbf{s}_{d+2}&\ldots&\mathbf{s}_{dn}\end{smallmatrix}]$, for $\mathbf{v}_i$ denoting the $d$ eigenvectors of $H$ and $\{\mathbf{s}_{d+2},\ldots,\mathbf{s}_{dn}\}$ is an arbitrary (orthonormal) basis for $\bar{\mathcal{S}}$. The representations of $X,Y,Z$ in this basis, given by $\tilde{X} = U_0^TXU_0$, $\tilde{Y} = U_0^TYU_0$ and $\tilde{Z} = U_0^TZU_0$, are as follows: $\tilde{X}$ has a single nonzero entry $\tilde{X}_{d+1,d+1} = a$, $\tilde{Y}$ is supported on its lower $(dn-(d+1))\times(dn-(d+1))$ block, $\tilde{Z}$ is supported on its $(d+1)$'th row and column except the first $d+1$ entries. Hence we get
\begin{equation}
\label{eq:QbasisRepresent}
\tilde{c}U_0^TT^*U_0 = \left[\begin{matrix}0_{d\times d}&&0_{d\times (dn-d)}\\&a&\tilde{\mathbf{z}}^T\\ 0_{(dn-d)\times d} &\tilde{\mathbf{z}} & \tilde{Y}\end{matrix}\right] \succeq 0
\end{equation}
where $\tilde{\mathbf{z}}$ denotes the nonzero entries of $\tilde{Z}$. By a Schur complement argument, (\ref{eq:QbasisRepresent}) implies 
\begin{equation}
\label{ZnormBnd}
\tilde{Y} - \frac{\tilde{\mathbf{z}}\tilde{\mathbf{z}}^T}{1-\tr(\tilde{Y})} \succeq 0 \ \Rightarrow \ 2\tr(Y)(1-\tr(Y)) \geq 2\|\tilde{\mathbf{z}}\|_2^2 = \|\tilde{Z}\|_F^2 = \|Z\|_F^2
\end{equation}
where we use $a = 1-\tr(Y)$ and $\tr(\tilde{Y}) = \tr(Y)$. Combining (\ref{eq:SumSqrs}), (\ref{eq:XnormBnd}), (\ref{ZnormBnd}) with the simple fact $\|Y\|_F \leq \tr(Y)$, we get the assertion of  Lemma~\ref{lem:TraceBnd}. 
\qquad\end{proof}
\newline The next step is to bound $\tr(Y)$. We provide the result in Lemma~\ref{lem:TrYBnd}, where, $L_G$ denotes the Laplacian of the graph $G_t$ and the parameter $\kappa$ is given by $\kappa = (\min_{(i,j)\in E_t}\|\mathbf{t}^0_i-\mathbf{t}_j^0\|_2^2)^{-1}$.
\begin{lemma}
\label{lem:TrYBnd}
\begin{equation}
\label{eq:TrYBnd}
\tr(Y) \leq \alpha_1\epsilon\|\tilde{c}T^*-T_0\|_F + \alpha_2\epsilon^2, 
\end{equation}
where $\alpha_1 = \frac{\sqrt{2}m}{\lambda_{d+2}(L^0)}$ and $\alpha_2 = (\frac{\kappa\sqrt{d}\|L_{G}\|_F}{m} + 1)\frac{\lambda_{n}(L_G)}{\lambda_{d+2}(L^0)}$.
\end{lemma}
\begin{proof}
We start with a (loose) bound on $\tr(Y)$, given by
\begin{equation}
\label{eq:TrYBnd1}
\tr(L^0(\tilde{c}T^*)) = \tr(L^0Y) \geq \tr(\lambda_{d+2}(L^0)Y) = \lambda_{d+2}(L^0)\tr(Y)
\end{equation}
Now, in order to bound $\tr(L^0(\tilde{c}T^*))$, we consider the following partitioning
\begin{subequations}
\begin{align}
\tr(L^0(\tilde{c}T^*)) &= \tr((L^0-L)(\tilde{c}T^*)) + \tilde{c}\tr(LT^*) \\
\label{eq:BndPart1}
&= \tr((L^0-L)(\tilde{c}T^*-T_0)) - \tr(LT_0) + \tilde{c}\tr(LT^*) \\
\label{eq:BndPart2}
&\leq \tr((L^0-L)(\tilde{c}T^*-T_0)) + (\tilde{c}\kappa-1)\tr(LT_0)
\end{align}
\end{subequations}
where $L$ is given by (\ref{eq:CostLaplacian}), and (\ref{eq:BndPart1}) follows from $\tr(L^0T_0) = 0$ and (\ref{eq:BndPart2}) holds by the feasibility of $\kappa T_0$ and the optimality of $T^*$ for the SDR. Now, we use the noise model to bound the terms in (\ref{eq:BndPart2}). For the first term, we get
\begin{subequations}
\begin{align}
&\noindent \tr((L^0 - L)(\tilde{c}T^* - T_0)) \leq \|L^0 - L\|_F\|\tilde{c}T^*-T_0\|_F\\
&\leq \sum_{i\sim j} \|(\gamma^0_{ij} + \epsilon_{ij})(\gamma^0_{ij} +\epsilon_{ij})^T - \gamma^0_{ij}(\gamma^0_{ij})^T\|_F\|\tilde{c}T^* - T_0\|_F  \\
&= \sum_{i\sim j} \left(\|\epsilon_{ij}\|_2^4 + 4\|\epsilon_{ij}\|_2^2\epsilon_{ij}^T\gamma_{ij}^0 + 2(\epsilon_{ij}^T\gamma_{ij}^0)^2 + 2\|\epsilon_{ij}\|_2^2\right)^{\frac{1}{2}}\|\tilde{c}T^* - T_0\|_F \\
\label{eq:BndPart13}
& = \sum_{i\sim j}\left(2\|\epsilon_{ij}\|_2^2 - \|\epsilon_{ij}\|_2^4/2\right)^{\frac{1}{2}}\|\tilde{c}T^* - T_0\|_F \\ 
\label{eq:TrYBnd15}
&\leq \epsilon\sqrt{2}m\|\tilde{c}T^* - T_0\|_F
\end{align}
\end{subequations}
where (\ref{eq:BndPart13}) follows from $\|\gamma_{ij}\|_2^2 = \|\gamma_{ij}^0 + \epsilon_{ij}\|_2^2 = 1 + 2\epsilon_{ij}^T\gamma_{ij}^0 + \|\epsilon_{ij}\|_2^2 = 1$. For the second term in (\ref{eq:BndPart2}), we first have
\begin{subequations}
\begin{align}
&\tr(LT_0) = \tr((L-L^0)T_0) \\
&= \sum_{i\sim j} \tr((\gamma_{ij}^0(\gamma_{ij}^0)^T - \gamma_{ij}\gamma_{ij}^T)\gamma^0_{ij}(\gamma_{ij}^0)^T)\|\mathbf{t}_i^0-\mathbf{t}_j^0\|_2^2 \\
&= -  \sum_{i\sim j} ((\epsilon_{ij}^T\gamma_{ij}^0)^2 + 2\epsilon_{ij}^T\gamma_{ij}^0)\|\mathbf{t}_i^0-\mathbf{t}_j^0\|_2^2 \\
&= \sum_{i\sim j} (\|\epsilon_{ij}\|_2^2 - \frac{\|\epsilon_{ij}\|_2^4}{4})\|\mathbf{t}_i^0-\mathbf{t}_j^0\|_2^2 \\
\label{eq:TrYBnd24}
&\leq \epsilon^2\lambda_{n}(L_G)
\end{align}
\end{subequations}
In order to bound the multiplier $\tilde{c}\kappa - 1$, we use the feasibility of $T^*$
\begin{subequations}
\begin{align}
&m \leq \sum_{i\sim j} \tr(C^{ij}T^*) = \tr((L_G\otimes I_d)T^*) \leq \sqrt{d}\|L_G\|_F\tr(T^*) \\
\label{eq:TryBnd2Cst}
&\Rightarrow \ \tilde{c}\kappa - 1 \leq \frac{\kappa}{\tr(T^*)} + 1 \leq \frac{\kappa\sqrt{d}\|L_G\|_F}{m} + 1
\end{align} 
\end{subequations}
\indent Finally, combining (\ref{eq:TrYBnd1}), (\ref{eq:BndPart2}), (\ref{eq:TrYBnd15}), (\ref{eq:TrYBnd24}) and (\ref{eq:TryBnd2Cst}), we obtain the claim (\ref{eq:TrYBnd}) of Lemma~\ref{lem:TrYBnd}.
\qquad\end{proof}
\newline\indent We now use Lemmas~\ref{lem:TraceBnd} and~\ref{lem:TrYBnd} to obtain our \emph{SDR Stability} result:
\begin{proof} (Proof of Theorem~\ref{thm:WeakStability})
The (second) inequality in (\ref{eq:TraceBnd}) and (\ref{eq:TrYBnd}) provide a quadratic inequality for $\|\tilde{c}T^* - T_0\|_F$, given by
\begin{equation}
\label{eq:QuadIneqNT}
\|\tilde{c}T^* - T_0\|_F^2 - 2\epsilon\alpha_1\|\tilde{c}T^* - T_0\|_F - 2\epsilon^2\alpha_2 \leq 0
\end{equation}
Examining the roots of this polynomial immediately yields
\begin{equation}
\label{eq:WeakStaPoly}
\delta(T^*,T_0) \leq \|\tilde{c}T^* - T_0\|_F \leq  \epsilon\left[\alpha_1 + \left(\alpha_1^2 + 2\alpha_2\right)^{\frac{1}{2}}\right]
\end{equation}
which was to be shown.
\qquad\end{proof}
\section{Construction of ADM Framework}
\label{Apdx:ADMdetails}
We first introduce some notation: $\mathbf{1}_m$ and $\mathbf{0}_m$ denote all-ones and all-zeros vectors of dimension $m$. $\mbox{S}^{m\times m}$ denotes the space of $m\times m$ symmetric matrices. For $X\in\R^{m\times m}$, $\mbox{vec}(X)$ is the vector in $\R^{m^2}$ that contains the columns of $X$, stacked each on top of the next in the order that they appear in the matrix, and $\mbox{mat}(\mathbf{x})$ is the matrix $X$ such that $\mathbf{x} = \mbox{vec}(X)$. We also use $0$ for the all-zeros matrix when its dimensions are obvious from the context. \\
\indent In order to rewrite the SDR in standard form, consider the matrix variable $X$ and the (non-negative) slack variable $\mathbf{\nu} \in \R^{m}$ satisfying
\begin{subequations}
\begin{align}
\label{eq:Xvar}
X &= \left[\begin{matrix} T  & 0 \\ 0& \Diag(\mathbf{\nu})\end{matrix}\right] \ ,\\ 
1 &= \tr(C^{ij}T) -  \nu_{ij} = \tr(A^{(1)}_{ij}X)\ , \ (i,j) \in E_t 
\end{align}
\end{subequations}
where $A^{(1)}_{ij}$ is the matrix corresponding to the linear functional $X \rightarrow \tr(C^{ij}T) -  \mathbf{\nu}_{ij}$. Let $M$ denote the number of ({\em a priori}) zero entries of $X$, i.e. $M = 2dnm + m(m-1)$. We can also replace~(\ref{eq:Xvar}) by a set of linear equalities in $X$ given by $\tr(A^{(2)}_{k}X) = 0$, for $k = 1,\ldots,M/2$. Hence, for $X = \left[\begin{smallmatrix}T & \star \\ \star &\star \end{smallmatrix}\right]$, we get 
\begin{equation}
\left\{ \begin{matrix} T\succeq0 \\ \tr(C^{ij}T) \geq 1, \ (i,j)\in E_t \end{matrix}\right\} \iff \left\{ \begin{matrix} X\succeq0 \\ \tr(A^{(1)}_{ij}X) = 1, \ (i,j)\in E_t \\ \tr(A^{(2)}_{k}X) = 0, \ k = 1,\ldots,M/2 \end{matrix}\right\} \nonumber
\end{equation}
Now, we can rewrite~(\ref{eq:SDRTransEst}) in standard form (ignoring the equality constraint $\tr(HT) = 0$, for now)
\begin{subequations}
\label{eq:StndrdSDP}
\begin{align}
\underset{{\scriptstyle X}}{\text{minimize}}
& \ \ \tr(WX)\\
\text{subject to}
& \ \ \mathcal{A}(X) = \mathbf{b}\\
& \ \ X \succeq 0
\end{align}
\end{subequations}
where $W = \left[\begin{smallmatrix} L & 0 \\ 0 & 0\end{smallmatrix}\right]$ (for $L$ given by (\ref{eq:CostLaplacian})), $\mathbf{b} = \left[\begin{smallmatrix}\mathbf{1}_{m} \\ \mathbf{0}_{M/2}\end{smallmatrix}\right]$ and the linear operator $\mathcal{A}$ satisfies $\mathcal{A}(X) = \left[\begin{smallmatrix} \mathcal{A}^{(1)}(X)\\ \mathcal{A}^{(2)}(X)\end{smallmatrix}\right]$ for $\mathcal{A}^{(1)}$ and $\mathcal{A}^{(2)}$ given by
\begin{subequations}
\label{eq:Aopers}
\begin{align}
& \mathcal{A}^{(1)}:\mbox{S}^{dn+m\times dn+m} \rightarrow \R^{m} \ ; \ \ \ \mathcal{A}^{(1)}(X) = [\begin{matrix} \tr(A^{(1)}_{1}X)\ldots \tr(A^{(1)}_{m}X)\end{matrix}]^T\\
& \mathcal{A}^{(2)}:\mbox{S}^{dn+m\times dn+m} \rightarrow \R^{M/2} \ ; \ \ \ \mathcal{A}^{(2)}(X) = [\begin{matrix} \tr(A^{(2)}_{1}X)\ldots \tr(A^{(2)}_{M/2}X)\end{matrix}]^T
\end{align}
\end{subequations}
Here, we let $A^{(i)}$ ($i=1,2$) denote the matrices satisfying $\mathcal{A}^{(i)}(X) = A^{(i)}\mbox{vec}(X)$ and, hence, are given by
\begin{subequations}
\label{eq:dualStndrdSDPAmats}
\begin{align}
& A^{(1)}\defeq \left[\begin{matrix} \mbox{vec}(A^{(1)}_{1}) & \ldots & \mbox{vec}(A^{(1)}_{m}) \end{matrix}\right]^T \in \R^{m\times(dn+m)^2} \\
& A^{(2)}\defeq \left[\begin{matrix} \mbox{vec}(A^{(2)}_{1}) & \ldots & \mbox{vec}(A^{(2)}_{M/2}) \end{matrix}\right]^T \in \R^{M/2\times(dn+m)^2}
\end{align}
\end{subequations}
We also let $A \defeq \left[\begin{smallmatrix} A^{(1)} \\ A^{(2)}\end{smallmatrix}\right]$ satisfying $\mathcal{A}(X) = A\mbox{vec}(X)$. Note that $A^{(1)}(A^{(2)})^T = 0$ and $A^{(2)}(A^{(2)})^T = I_{M/2}$ (after scaling $A^{(2)}_k$'s). Now, the dual of~(\ref{eq:StndrdSDP}) is 
\begin{subequations}
\label{eq:dualStndrdSDP}
\begin{align}
\underset{{\scriptstyle \mathbf{y}, S}}{\text{minimize}}
& \ \ -\mathbf{b}^T\mathbf{y}\\
\text{subject to}
& \ \ \mathcal{A}^*(\mathbf{y}) + S =  W\\
& \ \ S \succeq 0
\end{align}
\end{subequations}
where the operator $\mathcal{A}^*: \R^{m+M/2} \rightarrow \mbox{S}^{dn+m\times dn+m}$ is the adjoint of $\mathcal{A}$ and is defined by $\mathcal{A}^*(\mathbf{y}) \defeq \mbox{mat}(A^T\mathbf{y})$. Hence, the augmented Lagrangian function for the dual SDP~(\ref{eq:dualStndrdSDP}) is given by
\begin{equation}
\label{eq:AugLagran}
\mathcal{L}_{\mu}(\mathbf{y},S,X) = -\mathbf{b}^T\mathbf{y} + \tr(X(\mathcal{A}^*(\mathbf{y})+S-W)) + \frac{\mu}{2}\left\|\mathcal{A}^*(\mathbf{y})+S-W\right\|_F^2
\end{equation}
where $X\in\mbox{S}^{dn+m\times dn+m}$ and $\mu > 0$ is a penalty parameter. Here, we can obtain an alternating direction method (ADM) to minimize $\mathcal{L}_{\mu}(\mathbf{y},S,X)$ with respect to $\mathbf{y},S,X$ in an alternating fashion: Starting from an initial point $(\mathbf{y}^0,S^0,X^0)$, ADM sequentially solves the following problems at each iteration~\cite{WenADM}
\begin{subequations}
\begin{align}
\label{eq:ADMprblms1}
\mathbf{y}^{k+1} &\defeq \underset{{\scriptstyle \mathbf{y}}}{\text{argmin}} \ \mathcal{L}_{\mu}(\mathbf{y},S^k,X^k)\\
\label{eq:ADMprblms2}
S^{k+1} &\defeq \underset{{\scriptstyle S\succeq0}}{\text{argmin}} \ \mathcal{L}_{\mu}(\mathbf{y}^{k+1},S,X^k)
\end{align}
\end{subequations}
and finally updates (the dual Lagrange multiplier) $X$ by
\begin{equation}
\label{eq:ADMXupdate}
X^{k+1} \ \defeq \ X^k + \mu(\mathcal{A}^*(\mathbf{y}^{k+1})+S^{k+1}-W)
\end{equation}
We solve~(\ref{eq:ADMprblms1}) by setting $\nabla_{\mathbf{y}}{\mathcal{L}_{\mu}} = 0$, and obtain
\begin{equation}
\label{eq:yUpdate}
\mathbf{y}^{k+1} = -\left(\mathcal{A}\mathcal{A}^*\right)^{-1}\left(\frac{1}{\mu}\left(\mathcal{A}(X^k)-\mathbf{b}\right) + \mathcal{A}\left(S^k-W\right)\right)
\end{equation}
where the operator $\mathcal{A}\mathcal{A}^*$ satisfies $\mathcal{A}\mathcal{A}^*(\mathbf{y}) = \left[\begin{smallmatrix} A^{(1)}(A^{(1)})^T& 0 \\ 0 & \mbox{\scriptsize I}_{M/2}\end{smallmatrix}\right]\mathbf{y}$. Letting $\mathcal{B}:\mbox{S}^{dn+m\times dn+m} \rightarrow \R^{m}$ denote the linear operator
\begin{equation}
\label{eq:Boperator}
\mathcal{B}(X) = [\begin{matrix} \tr(C^1T)\ldots \tr(C^{m}T)\end{matrix}]^T \ , \ \mbox{for} \ \ \ X = \left[\begin{matrix} T & \star \\ \star & \star\end{matrix}\right]
\end{equation}
where $C^k$ is the constraint matrix of the $k$'th edge in $E_t$ (given by (\ref{eq:ConstLaplacian})), we get $\mathcal{A}^{(1)}(\mathcal{A}^{(1)})^* = \mathcal{B}\mathcal{B}^* + I$ (and hence, $A^{(1)}(A^{(1)})^T = BB^T + I_{m}$, where $B$ satisfies $\mathcal{B}(X) = B\mbox{vec}(X)$), making $\mathcal{A}\mathcal{A}^*$ invertible. Also, by rearranging the terms of $\mathcal{L}_{\mu}(\mathbf{y}^{k+1},S,X^k)$, (\ref{eq:ADMprblms2}) becomes equivalent to
\begin{equation}
\label{eq:SUpdate1}
S^{k+1} = \underset{{\scriptstyle S\succeq0}}{\text{argmin}} \left\| S - \Theta^{k+1}\right\|_F
\end{equation}
where $\Theta^{k+1}\defeq W - \mathcal{A}^*(\mathbf{y}^{k+1})-\frac{1}{\mu}X^k$. Hence, we get the solution $S^{k+1} = U_+\Sigma_+U_+^T$ where 
\begin{equation}
\label{eq:Hdecomp}
\Theta^{k+1} = \left[\begin{matrix}U_+ & U_-\end{matrix}\right]\left[\begin{matrix}\Sigma_+ & 0 \\ 0 & \Sigma_-\end{matrix}\right]\left[\begin{matrix}U_+^T \\ U_-^T\end{matrix}\right]
\end{equation}
is the spectral decomposition of $\Theta^{k+1}$, and $\Sigma_+$ denotes the diagonal matrix of positive eigenvalues of $\Theta^{k+1}$. Lastly, using (\ref{eq:ADMXupdate}), $X^{k+1}$ is given by
\begin{equation}
\label{eq:ADMXupdate2}
X^{k+1} = -\mu U_-\Sigma_-U_-^T
\end{equation}

\indent When implemented as is, the $\mathbf{y},S$ and $X$ updates of each iteration of ADM are extremely costly. This computational cost stems from the increase in the dimensions of the variables involved (after stating the primal and dual SDPs in standard form), and manifests itself via the computation of the inverse of $\mathcal{AA^*}$, the $\mathbf{y}$ update~(\ref{eq:yUpdate}) and the decomposition~(\ref{eq:Hdecomp}). However, considering the convergence analysis of ADM given in~\cite{WenADM}, since $\mathcal{A}$ is full-rank and the Slater condition (clearly) holds for the SDP~(\ref{eq:StndrdSDP}) (i.e., there exists $\hat{X}\succ 0$ such that $\mathcal{A}(\hat{X}) = \mathbf{b}$), ADM converges to a primal-dual solution $(X^*, \mathbf{y}^*, S^*)$ of (\ref{eq:StndrdSDP}) and (\ref{eq:dualStndrdSDP}) {\em irrespective of the initial point $(X^0, \mathbf{y}^0, S^0)$}. Here, we make the following crucial observation: Starting from $X^0$ and $S^0$ having the form
\begin{equation}
\label{eq:initXSform}
X^0 = \left[\begin{matrix} T^0  & 0 \\ 0& \Diag(\mathbf{\nu}^0)\end{matrix}\right] \ , \ \ S^0 = \left[\begin{matrix} R^0  & 0 \\ 0& \Diag(\mathbf{\eta}^0)\end{matrix}\right] 
\end{equation}
i.e., satisfying $\mathcal{A}^{(2)}(X^0) = \mathcal{A}^{(2)}(S^0) = \mathbf{0}_{M/2}$, $\mathbf{y}^k$ becomes essentially $m$ dimensional (i.e., $\mathbf{y}^k_i = 0$ for $i = m + 1,\ldots, m + M/2$) and $X^k$, $S^k$ preserve their form (i.e. $\mathcal{A}^{(2)}(X^k) = \mathcal{A}^{(2)}(S^k) = \mathbf{0}_{M/2}$), for each $k\geq1$. To see this, consider $X^k$ and $S^k$ satisfying (\ref{eq:initXSform}) for some $T^k\succeq0$, $\mathbf{\nu}^k\geq0$, and $R^k\succeq0$, $\mathbf{\eta}^k$. Then, by (\ref{eq:yUpdate}), we have
\begin{subequations}
\label{eq:Revyupdate}
\begin{align}
\mathbf{y}^{k+1} &= -\left[\begin{matrix}\left(A^{(1)}(A^{(1)})^T\right)^{-1} & 0\\ 0 & I_{M/2}\end{matrix}\right]\left(\frac{1}{\mu}\left[\begin{matrix}\mathcal{A}^{(1)}(X^k)-\mathbf{1}_m\\ \mathbf{0}_{M/2}\end{matrix}\right]+ \left[\begin{matrix}\mathcal{A}^{(1)}\left(S^k-W\right)\\ \mathbf{0}_{M/2}\end{matrix}\right]\right)\\
&= -\left[\begin{matrix} \left(\tilde{\mathcal{B}}\tilde{\mathcal{B}}^* + I\right)^{-1}\left(\frac{1}{\mu}(\tilde{\mathcal{B}}(T^k)-\mathbf{\nu}^k-\mathbf{1}_{m}) + \tilde{\mathcal{B}}(R^k-L)-\mathbf{\eta}^k\right) \\ \mathbf{0}_{M/2}\end{matrix}\right] \\
& =  \left[\begin{matrix}\mathbf{z}^{k+1}\\ \mathbf{0}_{M/2}\end{matrix}\right]
\end{align}
\end{subequations}
where, $\tilde{\mathcal{B}}$ is the (trivial) restriction of $\mathcal{B}$ to $\mbox{S}^{dn\times dn}$, i.e. $\tilde{\mathcal{B}}(T) = \mathcal{B}(X)$ for $X = \left[\begin{smallmatrix}T & \star \\ \star &\star \end{smallmatrix}\right]$. For $\Theta^{k+1} = W - \mathcal{A}^*(\mathbf{y}^{k+1})-\frac{1}{\mu}X^k$, we obtain
\begin{subequations}
\label{eq:RevHupdate}
\begin{align}
\Theta^{k+1} &= \left[\begin{matrix}L & 0\\ 0 & 0\end{matrix}\right] - \mbox{mat}\left((A^{(1)})^T\mathbf{z}^{k+1}\right) - \frac{1}{\mu}\left[\begin{matrix}T^k & 0 \\ 0 & \Diag(\mathbf{\nu^k})\end{matrix}\right] \\ 
& = \left[\begin{matrix}L - \frac{1}{\mu}T^k-\tilde{\mathcal{B}}^*(\mathbf{z}^{k+1}) & 0 \\ 0 & \Diag(\mathbf{z}^{k+1}-\frac{1}{\mu}\mathbf{\nu^k})\end{matrix}\right] 
\end{align}
\end{subequations}
Here, we let $F^{k+1} = \left[\begin{smallmatrix}V_+ & V_-\end{smallmatrix}\right]\left[\begin{smallmatrix}D_+ & 0 \\ 0 & D_-\end{smallmatrix}\right]\left[\begin{smallmatrix}V_+^T \\ V_-^T\end{smallmatrix}\right]$ be the spectral decomposition of $F^{k+1} \defeq L - \frac{1}{\mu}T^k-\tilde{\mathcal{B}}^*(\mathbf{z}^{k+1})$, where $D_+$ denotes the diagonal matrix of positive eigenvalues of $F^{k+1}$. Then, using (\ref{eq:SUpdate1}), (\ref{eq:Hdecomp}) and (\ref{eq:ADMXupdate2}), the $S^{k+1}$ and $X^{k+1}$ updates are given by
\begin{subequations}
\label{eq:RevSXupdate}
\begin{align}
S^{k+1} &= \left[\begin{matrix} V_+D_+V_+^T & 0 \\ 0 & \Diag\left(\max\left\{\mathbf{z}^{k+1} - \frac{1}{\mu}\mathbf{\nu}^k,\mathbf{0}_{m}\right\}\right) \end{matrix}\right] \\
X^{k+1} &= \left[\begin{matrix} -\mu V_-D_-V_-^T & 0 \\ 0 & -\mu\Diag\left(\min\left\{\mathbf{z}^{k+1} - \frac{1}{\mu}\mathbf{\nu}^k,\mathbf{0}_{m}\right\}\right) \end{matrix}\right]
\end{align}
\end{subequations}

As a result, we can rewrite the updates~(\ref{eq:ADMprblms1}), (\ref{eq:ADMprblms2}) and~(\ref{eq:ADMXupdate}) (or equivalently, (\ref{eq:yUpdate}), (\ref{eq:SUpdate1}) and~(\ref{eq:ADMXupdate2})), in terms of the new variables $\mathbf{z},\mathbf{\nu},\mathbf{\eta}\in\R^{m}$ and $R,T\in\R^{dn\times dn}$, as
\begin{subequations}
\label{eq:newADMupdates}
\begin{align}
\mathbf{z}^{k+1} &\defeq -\left(\tilde{\mathcal{B}}\tilde{\mathcal{B}}^* + I\right)^{-1}\left(\frac{1}{\mu}(\tilde{\mathcal{B}}(T^k)-\mathbf{\nu}^k-\mathbf{1}_{m}) + \tilde{\mathcal{B}}(R^k-L)-\mathbf{\eta}^k\right) \ ,\\
R^{k+1} &\defeq V_+D_+V_+^T \ , \\
\mathbf{\eta}^{k+1} &\defeq \max\left\{\mathbf{z}^{k+1} - \frac{1}{\mu}\mathbf{\nu}^k,\mathbf{0}_{m}\right\} \ ,\\
T^{k+1} &\defeq -\mu V_-D_-V_-^T \ , \\
\mathbf{\nu}^{k+1} &\defeq -\mu \min\left\{\mathbf{z}^{k+1} - \frac{1}{\mu}\mathbf{\nu}^k,\mathbf{0}_{m}\right\} \ .
\end{align}
\end{subequations}
Hence, we are able to overcome the excess computational cost induced by rewriting~(\ref{eq:SDRTransEst}) in standard form and maintain convergence. Also, the linear constraint $\tr(HT) = 0$ of (\ref{eq:SDRTransEst}) can be trivially incorporated into the ADM framework in the following way: Let $\tilde{B}\in\R^{m\times d^2n^2}$ denote the matrix corresponding to the linear operator $\tilde{\mathcal{B}}$, i.e. $\tilde{\mathcal{B}}(T) = \tilde{B}\mbox{vec}(T)$, and let $\tilde{B}_l$ denote the $l$'th row of $\tilde{B}$ given by $\tilde{B}_l = \mbox{vec}(C^l)^T$, for $l\in E_t$. Then, since $\tilde{\mathcal{B}}^*(\mathbf{z}^{k+1}) = \mbox{mat}(\tilde{B}^T\mathbf{z}^{k+1}) = \sum_l \mathbf{z}^{k+1}_l C^l$, we get
\begin{subequations}
\label{eq:RevZeroTrV}
\begin{align}
\tr(HF^{k+1}) &= \tr(HL) - \frac{1}{\mu}\tr(HT^{k}) - \sum_{l=1}^m \mathbf{z}^{k+1}_l\tr(HC^l) \\
& = - \frac{1}{\mu}\tr(HT^{k})
\end{align}
\end{subequations}
Here, if we have $\tr(HT^{k}) = 0$, we get $\tr(HF^{k+1}) = 0$, which, together with $T^k\succeq0$, implies that the eigenvectors of $H$ are in the nullspace of $F^{k+1}$. Hence, we obtain $\tr(HT^{k+1}) = 0$. As a result, if we choose $T^0$ such that $\tr(HT^0) = 0$, by induction we get $\tr(HT^k) = 0, \forall k$, hence for the solution $T^*$ of ADM (as requested in~(\ref{eq:SDRTransEst})). 

Moreover, we now show that the operator $\tilde{\mathcal{B}}\tilde{\mathcal{B}}^*$ has a simple structure allowing efficient computation of the inverse $(\tilde{\mathcal{B}}\tilde{\mathcal{B}}^* + I)^{-1}$.
\begin{lemma}
\label{lem:InverBBT}
Let $\tilde{B}\in\R^{m\times d^2n^2}$ denote the matrix corresponding to the linear operator $\tilde{\mathcal{B}}$, i.e. $\tilde{\mathcal{B}}(T) = \tilde{B}\mbox{vec}(T)$. Then the inverse $(\tilde{\mathcal{B}}\tilde{\mathcal{B}}^* + I)^{-1}$ is given by
\begin{subequations}
\label{eq:BBTinv}
\begin{align}
\left(\tilde{\mathcal{B}}\tilde{\mathcal{B}}^* + I\right)^{-1}(\mathbf{z}) &= \left(\tilde{B}\tilde{B}^T+I_{m}\right)^{-1}\mathbf{z} \\
&=\left(V_{\tilde{B}}\left(D_{\tilde{B}}+I_{n}\right)^{-1}V_{\tilde{B}}^T + \frac{1}{2d+1}\left(I_{m} - V_{\tilde{B}}V_{\tilde{B}}^T\right)\right)\mathbf{z}
\end{align}
\end{subequations}
where $V_{\tilde{B}}$ and $D_{\tilde{B}}$ denote the matrix of eigenvectors and the diagonal matrix of eigenvalues of $\tilde{B}\tilde{B}^T$, corresponding to its $n$ largest eigenvalues.
\end{lemma}
\begin{proof}
Let $k = (i_k,j_k)$ denote the $k$'th edge in $E_t$ and $\tilde{B}_k$ denote the $k$'th row of $\tilde{B}$ given by $\tilde{B}_k = \mbox{vec}(C^k)^T$, where $C^k$ is given by (\ref{eq:ConstLaplacian}). Then, the $(k,l)$'th entry of $\tilde{B}\tilde{B}^T$ is given by
\begin{equation}
\label{eq:entryBBT}
\tilde{B}\tilde{B}^T_{k,l} = \begin{cases} &0 \ \ \ \ ,\ \ \ \{(i_k,i_l),(i_k,j_l),(j_k,i_l),(j_k,j_l)\}\cap E_t = \emptyset \\ &d \ \ \ \ , \ \ \ \{(i_k,i_l),(i_k,j_l),(j_k,i_l),(j_k,j_l)\}\cap E_t \neq \emptyset \ \mbox{and} \ k\neq l \\ & 4d \ \  \ , \ \ \ k = l\end{cases}
\end{equation}
i.e., $\tilde{B}\tilde{B}^T = dM_{G}^TM_{G} + 2dI_{m}$, where $M_{G}\in\R^{n\times m}$ is the (vertex-edge) incidence matrix of $G_t$ (see, e.g.,~\cite{GraphSpecThy}). Hence, the spectral decomposition of $\tilde{B}\tilde{B}^T$ is given by
\begin{equation}
\label{eq:BBtSpecDec}
\tilde{B}\tilde{B}^T = \left[\begin{matrix}V_{\tilde{B}} & U_{\tilde{B}}\end{matrix}\right]\left[\begin{matrix}D_{\tilde{B}} & 0 \\ 0 & 2dI_{m-n}\end{matrix}\right]\left[\begin{matrix}V_{\tilde{B}}^T \\ U_{\tilde{B}}^T\end{matrix}\right]
\end{equation}
which clearly implies the claim of the lemma.
\qquad\end{proof}
\newline\indent As a result, in order to compute $(\tilde{\mathcal{B}}\tilde{\mathcal{B}}^* + I)^{-1}$, we need to compute only the $n$ largest eigenvalues and the corresponding eigenvectors of the sparse matrix $\tilde{B}\tilde{B}^T$. 
\section{Proof of Proposition~\ref{propo:ExactRecoveryDist}}
\label{Apdx:ExactRecoveryDist}
We prove the result in two steps: First we prove exact recovery of patch signs by (\ref{eq:PatchReg}) and EVM, then we prove exact recovery of global locations by (\ref{eq:JointlyScTrSynch}) (given the noiseless signs). \\
Consider a pair of patches $P_i$ and $P_j$, $(i,j)\in E_P$. Using (\ref{eq:LocalCoorSysRepres}), we obtain (where, we tamely assume $|c^i| > 0$, $\forall i\in V_P$, and $\mathbf{t}_k = \mathbf{t}_l \iff k = l$, $\forall k,l \in \bigcup P_i$)
\begin{equation}
\label{eq:LocalCoorSysRepresInv}
\mathbf{t}_k^i = \frac{c^j}{c^i}\mathbf{t}_k^j + \frac{\mathbf{t}^j-\mathbf{t}^i}{c^i}, \ \ k\in P_i\cap P_j
\end{equation}
Observe that, (\ref{eq:LocalCoorSysRepresInv}) implies, $c^{ij} = c^j/c^i$ and $\mathbf{t}^{ij} = (\mathbf{t}^j-\mathbf{t}^i)/c^i$ comprise a minimizer of the pairwise patch registration step (\ref{eq:PatchReg}), with a cost equal to $0$. Hence, every minimizer of (\ref{eq:PatchReg}) has a cost equal to $0$, i.e. is a solution of the linear equations
\begin{equation}
\label{eq:PatchRegSolns}
\mathbf{t}_k^i = c^{ij}\mathbf{t}_k^j + \mathbf{t}^{ij}, \ \ k\in P_i\cap P_j
\end{equation}
However, since $|P_i\cap P_j|\geq2$, $c^{ij}$ can be uniquely determined by simply selecting $2$ equations from (\ref{eq:PatchRegSolns}), subtracting one equation from the other (to eliminate $\mathbf{t}^{ij}$) and solving for $c^{ij}$ (note that, since we assume $\mathbf{t}_k \neq \mathbf{t}_l$, for $k\neq l$, and $|c^i| > 0$, $\forall i\in V_P$, we get $\mathbf{t}_k^i \neq \mathbf{t}_l^i$, for $k\neq l$, which allows us to solve for $c^{ij}$). Also, $\mathbf{t}^{ij}$ is uniquely determined by substituting the solution for $c^{ij}$ in one of the equations in (\ref{eq:PatchRegSolns}). As a result, $c^{ij} = c^j/c^i$ and $\mathbf{t}^{ij} = (\mathbf{t}^j-\mathbf{t}^i)/c^i$ is the unique solution of (\ref{eq:PatchReg}), for each $(i,j)\in E_P$, which yields the (unique) pairwise sign solution $z^{ij} = \mbox{sign}(c^j/c^i) = z^iz^j$ (i.e., (\ref{eq:PatchReg}) recovers the pairwise signs exactly for all $(i,j)\in E_P$). Given the noiseless pairwise signs $\{z^{ij}\}_{(i,j)\in E_P}$ defined on the connected graph $G_P$, EVM~\cite{AmitMihailSensors} recovers the exact patch signs $z^i$ (up to a global sign $z$), i.e. we obtain $\hat{z}^i = zz^i$ (for arbitrary $z \in\{-1,+1\}$).\\
Now, we define $\beta \defeq (\min_i |c^i|)^{-1}$, $\mathbf{x}_k \defeq z\beta\mathbf{t}_k$, $s^i \defeq \beta|c^i|$ and $\mathbf{u}^i\defeq z\beta\mathbf{t}^i$, for each $i\in V_P$ and $k\in \bigcup P_i$. Here, $\{\mathbf{x}_k, s^i, \mathbf{u}^i\}$ satisfy
\begin{equation}
\label{eq:SolnConstr}
\mathbf{x}_k-(s^i\hat{z}^i\mathbf{t}_k^i + \mathbf{u}^i) = z\beta(\mathbf{t}_k - (c^i\mathbf{t}_k^i + \mathbf{t}^i)) = \mathbf{0}, \ \forall \ k\in P_i, \ i\in V_P
\end{equation} 
where, the last equality follows from (\ref{eq:LocalCoorSysRepres}). We note that (\ref{eq:SolnConstr}), together with the feasibility $s^i\geq1$, implies that $\{\mathbf{x}_k, s^i, \mathbf{u}^i\}_{k\in P_i, i\in V_P}$ is a minimizer of (\ref{eq:JointlyScTrSynch}), with a cost equal to $0$. We now show that, for any minimizer $\{\mathbf{y}_k, r^i, \mathbf{w}^i\}_{k\in P_i, i\in V_P}$ of (\ref{eq:JointlyScTrSynch}), $\{\mathbf{y}_k\}_{k\in\bigcup P_i}$ is congruent $\{\mathbf{x}_k\}_{k\in\bigcup P_i}$, which is itself congruent to $\{\mathbf{t}_k\}_{k\in\bigcup P_i}$. Since $\{\mathbf{y}_k, r^i, \mathbf{w}^i\}_{k\in P_i, i\in V_P}$ (with $r^i \geq 1$, $\forall i\in V_P$) must have the cost value $0$, for each $i\in V_P$ we get
\begin{equation}
\label{eq:AltSolnEqns}
\mathbf{t}_k^i = \frac{\mathbf{y}_k - \mathbf{w}^i}{zz^ir^i} = \frac{\mathbf{x}_k - \mathbf{u}^i}{zz^is^i} \ \Rightarrow \ \mathbf{y}_k = \frac{r^i}{s^i}\mathbf{x}_k + \mathbf{w}^i - \frac{r^i}{s^i}\mathbf{u}^i, \ \ \forall k\in P_i
\end{equation}
Here, for each $(i,j)\in E_P$, we can find at least two separate $k \in P_i\cap P_j$ (say, $k_1$ and $k_2$), for which (\ref{eq:AltSolnEqns}) implies
\begin{subequations}
\begin{align}
\label{eq:AltSolnEqnsTwoPatches1}
&\mathbf{y}_{k_1} - \mathbf{y}_{k_2} = \frac{r^i}{s^i}\left(\mathbf{x}_{k_1} - \mathbf{x}_{k_2}\right)  = \frac{r^j}{s^j}\left(\mathbf{x}_{k_1} - \mathbf{x}_{k_2}\right) \\
\label{eq:AltSolnEqnsTwoPatches2}
&\Rightarrow \ \frac{r^i}{s^i} = c \ , \ \forall i\in V_P \ \Rightarrow \ \mathbf{w}^i - \frac{r^i}{s^i}\mathbf{u}^i = \mathbf{t} \ , \ \forall i\in V_P
\end{align}
\end{subequations}
where, in~(\ref{eq:AltSolnEqnsTwoPatches2}), the first implication follows from $r^i\geq1$, $\mathbf{x}_{k_1}\neq\mathbf{x}_{k_2}$ (since $\mathbf{x}_k$ are congruent to $\mathbf{t}_k$) and the connectivity of $G_P$ (i.e., since each pair results in a constant ratio of scales $r^i/s^i$, we can propagate this information to all the patches to conclude that this constant is the same for all patches), and also the second implication follows from substituting $r^i/s^i = c$ in (\ref{eq:AltSolnEqns}) and the connectivity of $G_P$ (again, allowing us to conclude that $\mathbf{w}^i - (r^i/s^i)\mathbf{u}^i$ is fixed for all patches). This completes the proof, since we just showed that, for any minimizer $\{\mathbf{y}_k, r^i, \mathbf{w}^i\}_{k\in P_i, i\in V_P}$ of (\ref{eq:JointlyScTrSynch}), we have 
\begin{equation}
\label{eq:ProofConclude}
\mathbf{y}_k = z(c\beta)\mathbf{t}_k + \mathbf{t} \ \ ; \ \ z\in \{-1,+1\}\ , \ \beta>0 \ , \ c\geq1 \ , \ \mathbf{t} \in \R^d 
\end{equation}
i.e., $\{\mathbf{y}_k\}_{k\in\bigcup P_i}$ is congruent to $\{\mathbf{t}_k\}_{k\in\bigcup P_i}$. \qquad\endproof \\ 

{\bf Acknowledgements.} \ \ The authors wish to thank Shahar Kovalsky and Mica Arie-Nachimson for valuable discussions, and for their technical support in data generation and $3$D structure reconstruction for the real data experiments. They also thank Federica Arrigoni for pointing out a typo in the statement of Theorem~\ref{thm:LamanConds} in an earlier version of this paper. O.~\"{O}zye\c{s}il and A.~Singer were partially supported by Award Number R01GM090200 from the NIGMS, by Award Number FA9550-12-1-0317 and FA9550-13-1-0076 from AFOSR, and by Award Number LTR DTD 06-05-2012 from the Simons Foundation. R.~Basri is supported in part by the Israel Science Foundation grant No. 764/10, The Israeli Ministry of Science, and the Citigroup Foundation. The vision group at the Weizmann Institute is supported in part by the Moross Laboratory for Vision Research and Robotics.

\bibliographystyle{siam}
\bibliography{SfMbib}

\end{document}